%% file: main.tex
\documentclass{article}

\PassOptionsToPackage{numbers}{natbib}

\usepackage[preprint]{neurips_2026}
\usepackage{microtype}
\usepackage{graphicx}
\usepackage{subcaption}
\usepackage{booktabs}
\usepackage{enumitem}
\usepackage{hyperref}
\usepackage{algorithm}
\usepackage[noend]{algorithmic}

\usepackage{amsmath}
\usepackage{amssymb}
\usepackage{mathtools}
\usepackage{amsthm}
\usepackage[most]{tcolorbox}

\usepackage[capitalize,noabbrev]{cleveref}

\theoremstyle{plain}
\newtheorem{theorem}{Theorem}[section]
\newtheorem{proposition}[theorem]{Proposition}
\newtheorem{lemma}[theorem]{Lemma}

\theoremstyle{definition}
\newtheorem{definition}[theorem]{Definition}
\newtheorem{assumption}[theorem]{Assumption}
\newtheorem{remark}[theorem]{Remark}

\usepackage{graphicx}
\usepackage{amsmath,amsthm,amssymb,amsfonts}
\usepackage{bm}

\usepackage{booktabs} 
\usepackage{threeparttable}
\usepackage{makecell}
\usepackage{multirow}

\usepackage{subcaption}
\usepackage[nice]{nicefrac}
\usepackage{pifont}
\usepackage{mathabx}
\usepackage[colorinlistoftodos,textwidth=2.1cm]{todonotes}

\usepackage{tcolorbox}
\usepackage{pifont}
\definecolor{mydarkgreen}{RGB}{39,130,67}
\definecolor{mydarkred}{RGB}{192,25,25}
\definecolor{mydarkblue}{RGB}{0,0,140}
\definecolor{lightorange}{RGB}{252,210,153}
\newcommand{\purple}{\color{p\purple}}

\usepackage{empheq}
\definecolor{darkgreen}{rgb}{0.00,0.5,0.00}

\usepackage{tcolorbox}
\usepackage{framed}
\usepackage{cleveref}
\crefname{assumption}{assumption}{assumptions}

\theoremstyle{definition}

\newcommand{\ve}{\varepsilon}
\newcommand{\supp}{\mathrm{supp}}

\usepackage{sidecap}

\input{macros}

\usepackage{tikz}
\usetikzlibrary{arrows.meta, positioning}
\usepackage{wrapfig}

\usepackage{wrapfig}

\usepackage[utf8]{inputenc} 
\usepackage[T1]{fontenc}    
\usepackage{hyperref}       
\usepackage{url}            
\usepackage{booktabs}       
\usepackage{amsfonts}       
\usepackage{nicefrac}       
\usepackage{microtype}      
\usepackage{xcolor}         

\title{Causal Evaluation of Membership Inference Attacks}

\author{%
  Mathieu Even \\
   Inria PreMeDICaL, Inserm, \\
   Montpellier, France \\
\texttt{mathieu.even@inria.fr}
\And
Clément Berenfeld \\
   Inria PreMeDICaL, Inserm, \\
   Montpellier, France \\
  \And
    Linus Bleistein \\
School of Computer and\\ Communication Science (EPFL)\\
 School of Life Sciences (EPFL)
\\
Lausanne, Switzerland
  \AND
  Tudor Cebere \\
   Inria PreMeDICaL, Inserm, \\
   Montpellier, France \\
  \And
  Julie Josse \\
   Inria PreMeDICaL, Inserm, \\
   Montpellier, France \\
  \And
  Aurélien Bellet \\
   Inria PreMeDICaL, Inserm, \\
   Montpellier, France \\
\texttt{aurelien.bellet@inria.fr}
}

\begin{document}

\maketitle

\begin{abstract}
  Membership Inference Attacks (MIAs) aim to distinguish training points (members) from unseen data (non-members), and are widely used to quantify memorization and assess privacy risks. Standard MIA evaluation requires repeated retraining, which is computationally costly for large models. One-run (single training with randomized data inclusion) and zero-run (post hoc evaluation) methods are often used instead, but their statistical validity remains unclear.
We address this gap by framing MIA evaluation as a causal inference problem, defining \emph{memorization as the causal effect of including a data point in the training set}. This novel formulation reveals and formalizes key sources of bias in existing protocols: one-run methods suffer from interference between jointly included points, while zero-run evaluations are additionally confounded by distribution shift between member and non-member evaluation data.
We derive causal analogues of standard MIA metrics and propose practical estimators for multi-run, one-run, and zero-run regimes with non-asymptotic consistency guarantees. We validate our approach in several settings, including pretrained and fine-tuned LLMs, showing that it enables reliable measurement of MIA performance without retraining and under distribution shift. Overall, our framework provides a principled foundation for privacy evaluation in modern AI systems.\looseness=-1
\end{abstract}

\section{Introduction}

The rapid deployment of large-scale machine learning models has intensified concerns regarding data privacy and intellectual property, as such models are susceptible to privacy attacks that can infer sensitive information about individual training data points \cite{shokri2017membership, carlini2019secret, carlini2021extracting, nasr2023scalable, barbero2025extracting, hayes2025strong}.
Consequently, data owners and regulators increasingly question whether unauthorized information has been added to training sets, potentially exposing personal or copyrighted content.
Recent guidance from data protection authorities emphasizes the use of attack-based evaluations to assess privacy leakage and to determine whether a trained model may itself constitute personal data under existing regulatory frameworks \citep{edpb}, while recent high-profile legal challenges concerning data provenance and usage have raised similar questions regarding whether specific data was used in model training \citep{noauthor_nyt_nodate,noauthor_concord_nodate}.\looseness=-1

\looseness=-1 In this context, \emph{Membership Inference Attacks} (MIAs) offer a compelling conceptual and practical
tool, as they probe the minimal form of information a model can retain about a data point: whether it was part of the training set \cite{carlini2022membership}.
MIAs provide a principled measure of memorization and unintended privacy leakage, revealing when a model's internal representations encode training-specific artifacts that distinguish members from non-members, even when direct data leakage is rare or difficult to observe.
MIAs also play a key role in auditing differential privacy \cite{jagielski2020auditing,nasr2023tight} and enabling other privacy attacks, including data reconstruction \citep{carlini2021extracting, carlini2019secret, nasr2023scalable, barbero2025extracting} and dataset inference \cite{maini_dataset_2021,maini_llm_nodate}.

The standard framework for evaluating MIAs, the \emph{multi-run} regime, involves retraining models hundreds or thousands of times 
with and without a given point \cite{shokri2017membership, yeom2018privacy, carlini2022membership, DBLP:conf/icml/ZarifzadehLS24}. While statistically rigorous, this approach is computationally prohibitive for modern large-scale models, prompting a shift toward more efficient alternatives. In \emph{one-run} approaches, a single model is trained on a randomized subset of data \cite{steinke2023privacy,mahloujifar2024auditing,Liu2025one_run_quantile_regression}, whereas \emph{zero-run} methods evaluate a deployed model post hoc \citep{shokri_enhanced,shi2023detecting,bertran2023scalable,meeus2024did}. 
The latter scenario is particularly relevant when an external evaluator cannot directly intervene in model training, which requires provider cooperation and trust, and cannot independently retrain the model due to withheld information about the training data or algorithm, as typically the case when using MIAs to assess the privacy risks of LLMs deployed by dominant tech companies.\looseness=-1

However, when moving from controlled retraining to more practical settings, the statistical validity of MIA evaluation becomes unclear. In particular, recent work shows that zero-run MIAs on LLMs suffer from \emph{systematic bias caused by distribution shift} in the non-member data used for evaluation \citep{duan2024membership,zhang2024membership,meeus2024sok}.
For example, in typical MIA evaluations, \emph{members} are documents published before the LLM training cutoff date—e.g., arXiv papers or Wikipedia articles that were almost certainly included in training—while \emph{non-members} are documents published afterward. These sets often differ in topic, style, and temporal context, causing MIA success in membership prediction to reflect distributional differences rather than true memorization, thereby overestimating privacy leakage \citep{meeus2024sok}.\looseness=-1

\textbf{Contributions.} In this work, we reframe the evaluation of MIAs as a causal inference problem, \emph{defining memorization as the causal effect of including a data point in the training set}.
This shift from a predictive to a causal perspective allows us to explicitly identify the assumptions required to interpret MIA performance as true memorization rather than an artifact of confounding factors, and enables the design of robust estimators that correct systematic biases in existing approaches. Specifically, our contributions are as follows:\looseness=-1
\begin{enumerate}[itemsep=1pt, parsep=1pt, topsep=0pt, partopsep=0pt, leftmargin=20pt
]
\vspace{-.2em}
    \item \textbf{Causal formalization of MIA evaluation}:
    We introduce a causal evaluation framework for MIAs that isolates the effect of a training point on the model, accounting for non-causal signals that can distort memorization estimates. Accordingly, we define causal counterparts to standard MIA metrics, shifting evaluation from correlation to causation.\looseness=-1

    \item \textbf{A causal taxonomy}: We analyze multi-run, one-run, and zero-run MIA evaluations from a causal lens, formalizing their error sources: \emph{interference} \citep{hudgens2008toward} between jointly inserted points in one-run settings, and \emph{confounding} \citep{rosenbaum1983central} from non-random assignments in zero-run scenarios.\looseness=-1
        
    \item \textbf{Principled estimators:} We propose practical estimators for causal MIA metrics that correct for distribution shift bias and come with non-asymptotic consistency guarantees. 
    As a byproduct, we obtain a new approach to causal effect estimation under random interference, a challenging setting \citep{cinelli2025challenges}, leveraging learning-theoretic tools such as algorithmic stability \citep{bousquet2002stability} in a novel way.\looseness=-1
        
    \item \textbf{Empirical validation:} Through experiments on synthetic data, ResNet on CIFAR-10, and both pretrained and fine-tuned LLMs,
    we show that our causal estimators provide practical and reliable solutions for post-hoc MIA evaluation, effectively correcting evaluation biases observed in prior work on large language models.
\end{enumerate}

\section{Background}

Here and throughout, a training algorithm $\cA$ is a (potentially randomized) function that maps a dataset $\cD \subset \cX$ to model parameters $\theta = \cA(\cD)$ in a parameter space $\Theta \subset \mathbb{R}^D$.

\subsection{Membership Inference Attacks (MIAs)}
\label{sec:different_settings}

\begin{figure*}[t]
\begin{minipage}[t]{0.39\textwidth}
  \begin{algorithm}[H]
  \small
    \caption{Multi-run}
    \label{alg:n}
    \begin{algorithmic}[1]
      \STATE \textbf{Input:} algorithm $\mathcal{A}$, dataset $\cD$, target distribution $\mathcal{P}_T$, number of samples $n$
      \FOR{$i \in [n]$}
        \STATE Sample $X_i \sim \mathcal{P}_T$
        \STATE {\color{darkgreen}$A_i \sim \mathrm{Bernoulli}(1/2)$}
        \STATE {\color{darkgreen}$\dtraini \gets \begin{cases}
        \cD \cup \{X_i\} & \text{\!\!\!\!if } A_i=1,\\
        \cD & \text{\!\!\!\!otherwise.}
        \end{cases}$}
        \STATE $\theta_i\leftarrow \cA(\dtraini)$
        \STATE $Y_i \gets s_{\rm MIA}(X_i,{\color{darkgreen}\theta_i})$
      \ENDFOR
      \STATE \textbf{Output:} $(X_i, A_i, Y_i)_{i\in[n]}$
    \end{algorithmic}
  \end{algorithm}
\end{minipage}\hfill
\begin{minipage}[t]{0.31\textwidth}
  \begin{algorithm}[H]
  \small
    \caption{One-run}
    \label{alg:one}
    \begin{algorithmic}[1]
      \STATE \textbf{Input:} algorithm $\cA$, dataset $\cD$, target distribution $\mathcal{P}_T$, number of samples $n$
      \FOR{$i \in [n]$}
        \STATE Sample $X_i \sim \mathcal{P}_T$
        \STATE {\color{darkgreen}$A_i \sim \mathrm{Bernoulli}(1/2)$}
      \ENDFOR
      \STATE {\color{darkred}$\dtrain \gets \cD \cup \{X_i : A_i = 1\}$}
      \STATE $\theta\leftarrow \cA(\dtrain)$
      \FOR{$i \in [n]$}
        \STATE $Y_i \gets s_{\rm MIA}(X_i,{\color{darkred}\theta})$
      \ENDFOR
      \STATE \textbf{Output:} $(X_i, A_i, Y_i)_{i\in[n]}$
    \end{algorithmic}
  \end{algorithm}
\end{minipage}\hfill
\begin{minipage}[t]{0.28\textwidth}
  \begin{algorithm}[H]
  \small
    \caption{Zero-run}
    \label{alg:zero}
    \begin{algorithmic}[1]
      \STATE \textbf{Input:} {\color{darkred}trained model $\theta\leftarrow \mathcal{A}(\cD_{\rm train})$}, datasets $\cD_T \subset \dtrain$ (of underlying distribution $\cP_T$) and $\cD_0 \subset \cX \setminus \dtrain$
      \STATE $(X_i)_{i\in[n]} \gets \cD_T \cup \cD_0$
      \FOR{$i \in [n]$}
        \STATE {\color{darkred}$A_i \gets \one_\set{X_i \in \cD_T}$}
        \STATE $Y_i \gets s_{\rm MIA}(X_i,{\color{darkred}\theta})$
      \ENDFOR
      \STATE \textbf{Output:} $(X_i, A_i, Y_i)_{i\in[n]}$
    \end{algorithmic}
  \end{algorithm}
\end{minipage}
\caption{MIA evaluation regimes studied in this work, with the MIA abstracted by its score function $s_{\rm MIA}$.
Green highlights the key steps of the multi-run setting; red indicates the successive modifications introduced when moving to the one-run and zero-run settings.}
\label{fig:algs_settings}
\end{figure*}

An MIA aims to determine whether a specific record $x\in\cX$ was included in a model's training set. Typical MIAs use a scoring function $s_{\rm MIA}(x, \theta) \in \R$ (e.g., negative loss or log-likelihood of $\theta$ on $x$) to compute a \textit{membership confidence score}, which is thresholded to predict membership.
\textbf{In this work, our focus is not on designing MIAs} (i.e., constructing a suitable $s_{\rm MIA}$), \textbf{but on the validity of the three evaluation regimes described in \Cref{fig:algs_settings}}. These regimes differ in how they collect the evidence $(X_i, A_i, Y_i)_{i\in[n]}$, where $X_i \in \cX$ is a data point, $A_i \in \{0,1\}$ is its membership in the training set, and $Y_i$ is its MIA score.
Evaluating MIAs requires specifying the distribution over which memorization is measured; throughout the paper, we denote this target distribution on $\cX$ by $\cP_T$.

In the \textbf{Multi-run (\Cref{alg:n})} regime, each sample $(X_i, A_i, Y_i)$ is generated from an independently trained model $\theta_i$, with $X_i \sim \cP_T$ and randomized membership $A_i$.
In the \textbf{One-run (\Cref{alg:one})} regime introduced by \citet{steinke2023privacy}, a single model $\theta$ is trained, with $n$ points drawn from $\cP_T$ 
randomly included in the training set $\dtrain$.
Finally, the \textbf{Zero-run (\Cref{alg:zero})} regime, widely used in recent studies on LLMs \citep{shi2023detecting,mattern-etal-2023-membership,meeus2024did}, evaluates MIAs post-hoc on a fixed model, 
without retraining or randomized membership. The training set $\dtrain$ is fixed and
unknown to the evaluator. Evaluation uses proxy datasets: $\cD_T \subset \dtrain$ (known members, e.g., pre-2023 Wikipedia) 
and $\cD_0 \subset \cX \setminus \dtrain$ (known non-members, e.g., post-release articles).


\textbf{Evaluation metrics.}
The evidence $(X_i, A_i, Y_i)_{i \in [n]}$ is used to compute standard MIA performance metrics, measuring memorization by how well the scores $Y_i$ separate members ($A_i = 1$) from non-members ($A_i = 0$):
\textbf{(i)} \textit{Membership advantage}: $\esp{Y_i \mid A_i=1} - \esp{Y_i \mid A_i=0}$;
\textbf{(ii)} \textit{AUC}: $\proba{Y_i > Y_j \mid A_i=1, A_j=0}$; \label{eq:AUC}
\textbf{(iii)} \textit{TPR at fixed FPR}: $\proba{Y_i \geq t_\alpha \mid A_i=1}$ with $t_\alpha$ the $(1-\alpha)$-quantile of $Y_i \mid A_i=0$.

\subsection{Causal Inference}
\label{sec:background_causal}

In causal inference \citep{hernan2010causal,rubin1974estimating,wager2024causal}, observations are tuples $(X_i,A_i,Y_i)$, where $X_i$ denotes covariates, $A_i\in\{0,1\}$ is a treatment indicator, and $Y_i$ is here a real-valued outcome.
Under the potential outcomes framework \citep{splawa1990application,ImbensRubin2015,wager2024causal}, each unit has two counterfactual outcomes $Y_i(0)$ and $Y_i(1)$, but only $Y_i=Y_i(A_i)$ is observed. Individual treatment effects $Y_i(1)-Y_i(0)$ are thus not identifiable, meaning they cannot be uniquely determined from the distribution of the observed data. Causal inference instead focuses on population-level quantities such as the average treatment effect (ATE) $\esp{Y_i(1)-Y_i(0)}$, which are identifiable under the following standard assumptions.\looseness=-1
%
%
\begin{assumption}[Standard causal assumptions]\label{hyp:standard_causal_assumptions}
For all $ i\in[n]$, \textit{(1) (SUTVA)} $Y_i = Y_i(A_i)$; the assignment mechanism satisfies either \textit{(2a) (Randomized)} $A_i \indep (Y_i(0),Y_i(1))$ or \textit{(2b) (Unconfoundedness)} $A_i \indep (Y_i(0),Y_i(1))|X_i$; and \textit{(3) (Overlap)} $0<\proba{A_i=1|X_i}<1$ almost surely.
\end{assumption}
SUTVA and overlap imply a unit's outcome does not depend on other units' assignments, and that each unit has a nonzero probability of receiving either treatment.
The assumption on the assignment mechanism depends on the setting.\looseness=-1

\textbf{RCT vs observational study.}
In randomized control trials (RCTs), treatment is assigned randomly to each unit, so the \textit{Randomized} assumption holds. In observational studies, treatment may correlate with covariates that also affect outcomes, violating this assumption. \textit{Unconfoundedness} addresses this by requiring that, conditional on observed covariates, treatment assignment is as good as random.\looseness=-1

\looseness=-1\textbf{Causal inference with interference.}
When the SUTVA assumption is violated, \textit{interference} occurs:
a unit's outcome may depend on other units' treatment assignments, for instance through 
shared resources, competition effects, or 
herd immunity in vaccine trials \citep{sobel2006randomized}.
Causal inference with interference thus requires additional formalism \citep{hudgens2008toward}.
Specifically, interference is modeled by generalized potential outcomes $Y_i(\bar a)$ indexed by the full assignment vector $\bar a\in\{0,1\}^n$, with observed outcome
$Y_i = Y_i(\bar A)$ for $ \bar A=(A_j)_{j\in[n]}$.
We define the marginal potential outcome as:
\begin{equation}\label{eq:PO}
Y_i(a) = Y_i(A_1,\dots,A_{i-1},a,A_{i+1},\dots,A_n)\,.
\end{equation}
In the absence of interference, $Y_i(\bar a)=Y_i(a_i)$, recovering the standard setting.
Standard causal assumptions under interference \citep{ogburn2024causal} that generalize \Cref{hyp:standard_causal_assumptions} are detailed in \Cref{app:causal_assumptions}.

\section{MIA Evaluation Under a Causal Lens}

Our first key contribution is to frame membership inference as a causal question: \textit{what is the effect of including a given data point in the training set on the learned model?} In this formulation, the \emph{treatment} $A_i$ is the inclusion of a data point $X_i$ in the training set, and the \emph{outcome} $Y_i$ is the membership confidence score (which depends on both the model and the data point $X_i$). Potential outcomes $Y_i(1)$ and $Y_i(0)$ correspond to the model's output on $X_i$ under two counterfactual worlds: one in which $X_i$ was used for training and one in which it was not. Since only one of these worlds is observed, individual effects are unidentifiable, but population-level causal effects can be estimated. 
Casting membership inference in this causal framework clarifies what is being measured and enables principled control of interference and confounding that would otherwise bias the evaluation.\looseness=-1

\subsection{Causal Interpretation of MIA Evaluation Regimes}
\label{sec:mia_causal_setting}

We now show that we can explicitly map the three scenarios corresponding to \Cref{alg:n,alg:one,alg:zero} to standard causal inference settings, with their corresponding causal graphs shown in \Cref{fig:causal_graph}.

\begin{wrapfigure}[9]{r}{0.5\columnwidth}
    \centering
    \vspace*{-.3cm}
    \includegraphics[width=\linewidth,trim=0 10pt 0 5pt, clip]{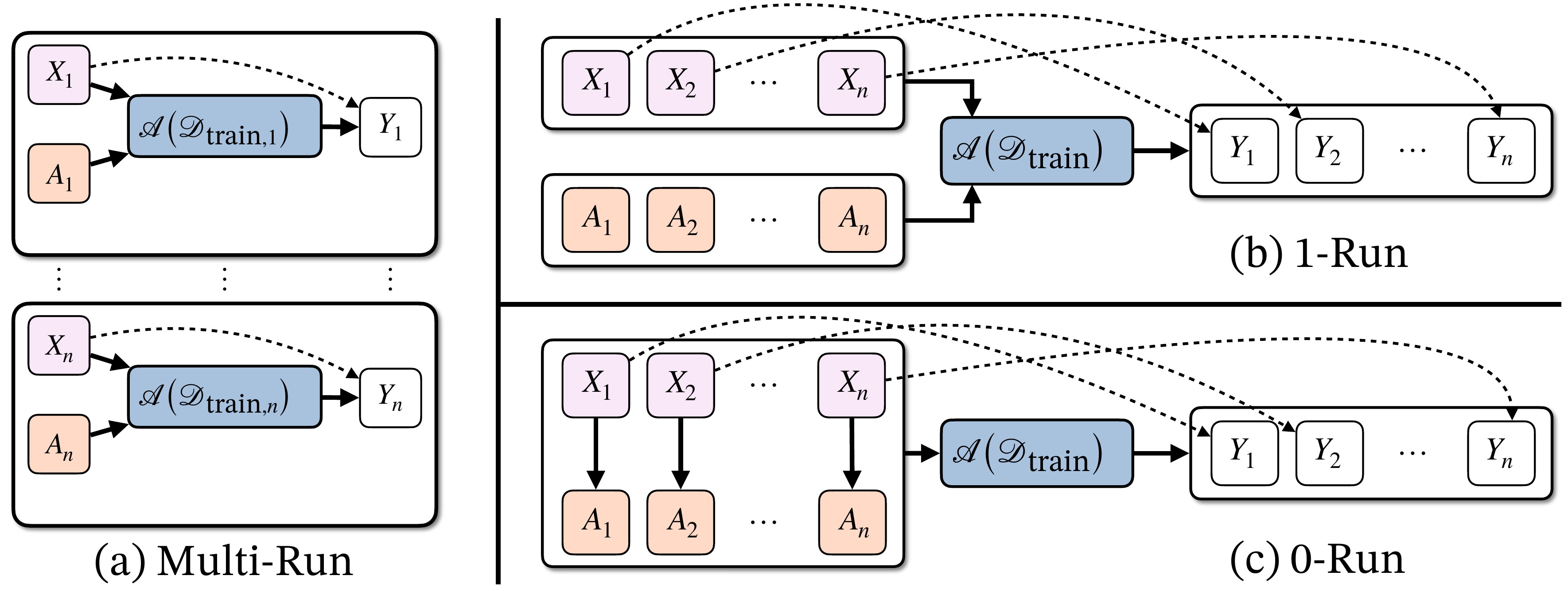}
    \caption{Causal graphs.}
    \label{fig:causal_graph}
\end{wrapfigure}

\textbf{Multi-run $\leftrightarrow$ RCT.} The treatment is randomized, $(X_i,A_i,Y_i)$ are i.i.d., corresponding to a RCT setting in causal inference where \Cref{hyp:standard_causal_assumptions} with randomized assignments holds.

\textbf{One-run $\leftrightarrow$ RCT with interference.}
While the randomized assignment property is preserved, the one-run setting trades computational efficiency (training a single model instead of $n$) for the presence of \emph{interference} between units. Indeed, the assignment $A_j$ of point $j\in[n]$ can affect $Y_i$ for $i\neq j$, since including or excluding $X_j$ changes the learned model used to compute all outcomes.
The one-run regime thus corresponds to an RCT with interference.
In the causal graph shown in \Cref{fig:causal_graph}b), interference arises through the shared mediator $\cA(\dtrain)$. Without this mediator, the graph would contain direct edges $A_i\rightarrow Y_j$ for all $i,j\in[n]$, corresponding to a complete interference graph.
While causal quantities are identifiable under standard assumptions (\Cref{hyp:RCT_interferences}) that here hold by design, additional assumptions will be required to ensure the consistency of our estimators, as we detail in Section~\ref{sec:interferences}.

\textbf{Zero-run $\leftrightarrow$ observational study with interference.}
This setting inherits interference from the one-run scenario and adds \emph{confounding}. As shown by the edge from $X_i$ to $A_i$ in \Cref{fig:causal_graph}c), membership is not assigned independently of the data points due to the distribution shift between members and non-members. This confounding factor, which is at the heart of misleading conclusions observed in memorization studies for LLMs \citep{meeus2024sok}, turns the problem into an \emph{observational study}. Framing the problem causally makes this shift explicit and will enable principled debiasing of evaluation metrics via confounder adjustment, as we describe in Section~\ref{sec:0run}.

\subsection{Causal MIA Evaluation Metrics}
\label{sec:mia_perf_causal}

\begin{table}
\centering
\caption{Classical MIA metrics and their causal counterparts.
\label{tab:eval_metrics}
}
\footnotesize
\label{tab:example}
\setlength{\tabcolsep}{0pt}
\begin{tabular}{cc}
\toprule
\textbf{Classical metrics} & \textbf{Causal metrics} \\
\midrule
\begin{tabular}[c]{@{}c@{}}
 Membership Advantage\\
$\esp{Y_i|A_i\!\!=\!\!1}\!-\!\esp{Y_i|A_i\!\!=\!\!0}$
\end{tabular}
   &
\begin{tabular}[c]{@{}c@{}}
Average Treatment Effect\\
$\tau_\mathrm{ATE}=\espT{Y_i(1)-Y_i(0)}$
\end{tabular}
\\
\midrule
\begin{tabular}[c]{@{}c@{}}
 $\mathrm{TPR}(t)\!\!=\!\!\proba{Y_i>t|A_i\!\!=\!\!1}$ \\
 $\mathrm{FPR}(t)\!\!=\!\!\proba{Y_i>t|A_i\!\!=\!\!0}$
\end{tabular}
   &
\begin{tabular}[c]{@{}c@{}}
$\tau_\mathrm{TPR}(t)\!\!=\!\!\probaT{Y_i(1)>t}$ \\
$\tau_\mathrm{FPR}(t)\!\!=\!\!\probaT{Y_i(0)>t}$\\
\end{tabular}
\\
\midrule
\begin{tabular}[c]{@{}c@{}}
AUC \\ 
Area under $\!\set{\mathrm{FPR}(t),\!\mathrm{TPR}(t)}$
\end{tabular}
   &
\begin{tabular}[c]{@{}c@{}}
Causal AUC\\ 
Area under $\!\set{\tau_\mathrm{FPR}(t),\!\tau_\mathrm{TPR}(t)}$
\end{tabular}
\\
\midrule
\begin{tabular}[c]{@{}c@{}}
TPR at fixed FPR $\alpha$ \\ 
$\proba{Y_i\geq t_\alpha|A_i\!\!=\!\!1}$\\
$t_\alpha$: $\!1\!\!-\!\!\alpha$ quantile of $Y_i|A_i\!\!=\!\!0$~~~
\end{tabular}
   &
\begin{tabular}[c]{@{}c@{}}
Causal TPR at fixed FPR $\alpha$ \\ 
$\tau_{\mathrm{TPR@FPR}}(\alpha)=\probaT{Y_i(1)\geq q_\alpha}$\\
$q_\alpha$: $\!1\!\!-\!\!\alpha$ quantile of $Y_i(0)|X_i\!\!\sim\!\!\cP_T$
\end{tabular}\\\bottomrule
\end{tabular}
\end{table}

When members ($A=1$) and non-members ($A=0$) follow different distributions, the evaluation metrics defined in \Cref{sec:different_settings} may conflate true memorization effects (i.e., differences in $Y_i$ versus $Y_j$ for $A_i \neq A_j$ even when $X_i \approx X_j$) with distribution shift effects, where differences in $Y$ arise simply because $X_i$ and $X_j$ differ.
To address this, we introduce causal metrics that enable meaningful MIA evaluation under both interference and distribution shift. 
\Cref{tab:eval_metrics} summarizes these metrics alongside their corresponding classical counterparts, where 
$\mathbb{E}_{\mathcal{P}_T}$ and $\mathbb{P}_{\mathcal{P}_T}$ denote expectations and probabilities over $X_i \sim \mathcal{P}_T$, where $\mathcal{P}_T$ is the target distribution. Under interference, all causal quantities implicitly depend on the dataset size $n$, as can be seen from \eqref{eq:PO}.\looseness=-1

In the next sections, for simplicity of presentation, we will focus on estimating the \textbf{Average Treatment Effect (ATE)}
   $\tau_\mathrm{ATE} = \espT{Y_i(1)-Y_i(0)}$,
which admits a direct interpretation as a \emph{causal membership advantage}. Extensions of our results to the other causal metrics are provided in \Cref{app:estimators_causal_metrics}.\looseness=-1

\begin{tcolorbox}[colback=gray!5!white,colframe=gray!85!white,top=2pt,bottom=2pt,title= \textbf{Insight \# 1}]
Our causal framework provides a principled way to characterize biases in MIA evaluation (interference and confounding) and to define performance metrics that disentangle distribution shift between member and non-member data from genuine memorization.
\end{tcolorbox}




\section{Consistency in the Multi-Run Regime}
\label{sec:n-run}

In the multi-run setting, identification follows from \Cref{hyp:standard_causal_assumptions}, which is satisfied by design.

\begin{proposition}\label{prop:n_run_identification}
    If $(X_i,A_i,Y_i)_{i\in[n]}$ are obtained using \Cref{alg:n}, causal evaluation metrics (\Cref{tab:eval_metrics}, right column) are identifiable and equal to classical evaluation metrics (\Cref{tab:eval_metrics}, left).
\end{proposition}

These properties then enable estimation of the causal evaluation metrics. Focusing on 
the ATE (the causal membership advantage) $\tau_\mathrm{ATE}$, 
letting $n_a = \#\{i:A_i=a\}$, we define the empirical estimator:
\begin{equation}\label{eq:estimator_ATE}
\begin{aligned}
    &\textstyle\hat \tau_\mathrm{ATE} = \frac{1}{n_1}\sum_{i:A_i=1}Y_i-\frac{1}{n_0}\sum_{i:A_i=0}Y_i \,.
    \end{aligned}
\end{equation}

\begin{proposition}\label{prop:mia_evaluation_n_run_ATE}
    If $(X_i,A_i,Y_i)_{i\in[n]}$ are obtained using \Cref{alg:n}, $\hat\tau_\mathrm{ATE}$ is a consistent estimator of $\tau_\mathrm{ATE}$ and if outcomes $Y_i$ lie in $[0,1]$,\footnote{
    Our results readily extend to outcomes in a bounded space of diameter $B$, via a rescaling of the results by $B$.
    } with probability $1-4e^{-t}$: $    |\hat\tau_\mathrm{ATE}-\tau_\mathrm{ATE}| \leq \sqrt{2t/n_1} + \sqrt{2t/n_0}$.
\end{proposition}

\section{Handling Interference in One-Run}
\label{sec:interferences}

We now consider the one-run setting, in which SUTVA in \Cref{hyp:standard_causal_assumptions} is violated due to \emph{interference}, so identification is no longer straightforward.
We remind that we work with generalized counterfactuals: for $\bar a \in \{0,1\}^n$, $Y_i(\bar a)$ is the outcome of unit $i\in[n]$ we would have observed had treatment assignments been $A_j=a_j$ for all $j\in [n]$, thereby explicitly modeling interference between data points. A generalization of \Cref{hyp:standard_causal_assumptions} to settings with interference is provided in \Cref{app:causal_assumptions} and holds by design, allowing us to generalize \Cref{prop:n_run_identification} to the one-run setting.

\begin{proposition}\label{prop:1_run_identification}
    If $(X_i,A_i,Y_i)_{i\in[n]}$ are obtained using \Cref{alg:one}, causal evaluation metrics (\Cref{tab:eval_metrics}, right) are identifiable and equal to classic evaluation metrics (\Cref{tab:eval_metrics}, left).
\end{proposition}

We consider the same estimators as in the multi-run setting, such as $\hat\tau_\mathrm{ATE}$ defined in \Cref{eq:estimator_ATE} for the ATE. 
Consistency, however, is no longer guaranteed, since interference induces dependence between samples.
As aptly noted by \citet[Section 3.2]{cinelli2025challenges}, “all no-interference applications are alike; each SUTVA violation violates SUTVA in its own way”, underscoring the absence of a one-size-fits-all solution for handling interference.
In our setting, as illustrated in \Cref{fig:causal_graph}b, interference arises because a single model is trained jointly on all members, inducing a complete interference graph over data points. We thus propose to impose assumptions directly on the training algorithm $\cA$, since the interference structure is intrinsic and unavoidable.
To that end, we draw on algorithmic stability notions from learning theory \cite{bousquet2002stability}, which formalize the limited sensitivity of a model's behavior to any single training point. In particular, error stability (Def.~\ref{def:error_stability}) bounds changes in expected loss, while uniform training stability (Def.~\ref{def:delta_interpolation}) controls the worst-case change in loss at any \emph{training} point.\looseness=-1

\begin{definition}[Error stability \cite{bousquet2002stability}]\label{def:error_stability}
    A learning algorithm $\cA$ is $\alpha$-error stable with respect to a loss $\cL:\Theta\times \cX\rightarrow\mathbb{R}_{\geq 0}$ if for all datasets $\cD=\set{x_1,\ldots,x_n}$, all $i\in[n]$ and $\cD'$ such that $\cD'=\set{x_1',\ldots,x_n'}$ with $x_j=x_j'$ if $j\ne i$, for $X \sim\cP_T$ independent of the rest, we have:
    \begin{equation*}
        |\E_\cA[\E_X[\cL(\cA(\cD),X)]]-\E_\cA[\E_X[\cL(\cA(\cD'),X)]]|\leq \alpha\,.
    \end{equation*}
\end{definition}


\begin{definition}[$\beta$-uniform training stability]\label{def:delta_interpolation}
        A learning algorithm $\cA$ is $\beta$-uniform training stable with respect to a loss $\cL$ if for all datasets $\cD=\set{x_1,\ldots,x_n}$, all $i\in[n]$ and $\cD'$ such that $\cD'=\set{x_1',\ldots,x_n'}$ with $x_j=x_j'$ if $j\ne i$, for any $k\in[n]\setminus\set{i}$, we have a.s.~over the randomness of $\cA$:
    \begin{equation*}
        |\cL(\cA(\cD),x_k)-\cL(\cA(\cD'),x_k)|\leq \beta\,.
    \end{equation*}
\end{definition}
\Cref{def:delta_interpolation} holds for uniformly stable algorithms with $\beta$ equal to the stability parameter. Interpolating models have $\beta=0$, and near-interpolating models (with small training loss) have small $\beta$.\looseness=-1

\begin{theorem}\label{thm:consistency_1_run}
    Assume that $(X_i,A_i,Y_i)_{i\in[n]}$ are obtained using \Cref{alg:one}, with $Y_i=-\cL(\theta,X_i)$ corresponding to the standard loss-based attack. Assume that $\cA$ is $\alpha$-error stable (\Cref{def:error_stability}), $\beta$-uniform training stable (\Cref{def:delta_interpolation}), and that outcomes $Y_i$ are in $[0,1]$ almost surely.
Then, with probability $1-Ce^{-t}$ for some numerical constant $C>0$:
    \begin{equation*}
        \textstyle|\hat \tau_\mathrm{ATE}-\tau_\mathrm{ATE}| =\cO\big( \sqrt{t/n} + \sqrt{n t\alpha^2} + \sqrt{n t\beta^2}\big)\,.
    \end{equation*} 
\end{theorem}

In the perfect interpolation regime ($\beta=0$), the first term $\sqrt{t/n}$ is expected and tends to 0.
For consistency, we thus need the second term $\sqrt{nt\alpha^2}$ to also be negligible, which is the case as long as the stability parameter $\alpha=o(1/\sqrt{n})$.
For regularized convex or strongly convex problems, this will always hold since then $\alpha$ is of order $1/N$ where $N$ is the total number of training points in $\dtrain$, leading to $\sqrt{nt\alpha^2} =\cO( 1/\sqrt{N})$ where $N\geq n$. In most applications, $n$ will be negligible in front of $N$, so the
second term becomes negligible in front of the first.
For nonconvex problems, error stability no longer holds in general, but for SGD training with smooth, Lipschitz losses and early-stopping, error stability holds with $\alpha$ of order $N^{-1+o(1)}$ \citep{hardt2016train}.
The term $\alpha\sqrt{n}$ will thus be of order $1/\sqrt{N}$ if $n=\cO(N^{1-o(1)})$.
In the general case, $\beta$ is also required to be small enough. This is the case for large overparametrized and near-interpolating networks, or for uniformly stable algorithms. 
We also provide a negative result for the underparametrized regime (i.e., $\frac{D}{n}\to0$) in \Cref{app:underparam}.

\begin{remark}[Extension to other MIAs]
Consistency results in this section are specific to loss-based MIAs, as they rely on standard learning-theoretic assumptions (\Cref{def:error_stability,def:delta_interpolation}). \Cref{thm:consistency_1_run} extends to other MIAs by assuming stability of the confidence score function $s_{\rm MIA}$.
\end{remark}

\begin{tcolorbox}[colback=gray!5!white,colframe=gray!85!white,top=2pt,bottom=2pt,title= \textbf{Insight \# 2}]
In the one-run regime, consistency of MIA metrics can be retained under stability assumptions on the training algorithm, which serves as the mediator of interference between data points.
\end{tcolorbox}

\section{Adjusting for Confounders in Zero-Run}
\label{sec:0run}

In the zero-run regime, the presence of confounders introduces an additional challenge for identifying and estimating causal evaluation metrics, due to the distributional shift between member data ($\cD_T$) and non-member data ($\cD_0$).
Indeed, the problem departs from a RCT and becomes an observational study, since
$\cP_{X_i|A_i=1}\ne\cP_{X_i|A_i=0}$.
As is standard in causal inference, observational studies cannot be conducted without assumptions: the identifiability of causal estimands (such as $\tau_\mathrm{ATE}$) depends on them.
We thus assume that any member would have had a non-zero probability of being non-member.

\begin{assumption}[Domain overlap]\label{hyp:distribution_shift}
$\cD_T$ and $\cD_0$ consist of independent samples from distributions $\cP_T$ and $\cP_0$, respectively, with $\frac{\dd \cP_T}{\dd \cP_0}(x) < \infty$ for all $x \in \mathrm{Supp}(\cP_T)$.
\end{assumption}

Under this assumption, \textit{overlap} holds (\Cref{hyp:standard_causal_assumptions}) together with identifiability.
Note that under \Cref{hyp:distribution_shift} we have $\cP_T=\cP_{X_i|A_i=1}$ and $\cP_0=\cP_{X_i|A_i=0}$.


\begin{proposition}\label{prop:0_run_identification}
    Under \Cref{hyp:distribution_shift}, if $(X_i,A_i,Y_i)_{i\in[n]}$ are obtained according to \Cref{alg:zero}, the causal evaluation metrics are identifiable.
    Furthermore,
$        \tau_\mathrm{ATE} = \frac{\esp{A_iY_i}}{\proba{A_i=1}} - \E\big[ \frac{\pi(X_i)(1-A_i)}{1-\pi(X_i)} Y_i\big]\proba{A_i=1}^{-1}$,
where $        \pi(x) = \proba{A_i=1|X_i=x}\,,\, \forall x\in\cX
$.
\end{proposition}

\Cref{prop:0_run_identification} shows that under distribution shift, the left and right columns of \Cref{tab:eval_metrics} no longer coincide. Consequently, estimators of standard evaluation metrics (left column) yield biased assessments of MIA performance. In particular, $\hat \tau_\mathrm{ATE}$ in \eqref{eq:estimator_ATE} estimates $\esp{Y_i\mid A_i=1}-\esp{Y_i\mid A_i=0}$, which differs from the \emph{causal} membership advantage $\tau_\mathrm{ATE}$ because $\cP_{Y_i\mid A_i=0}\neq \cP_{Y_i(0)\mid A_i=1}$, due to the distribution shift $\cP_0=\cP_{X_i|A_i=0}\ne\cP_{X_i|A_i=1}=\cP_T$.
To correct this confounding bias and recover $\tau_\mathrm{ATE}$, \Cref{prop:n_run_identification} naturally suggests an \emph{inverse probability weighting} (IPW) estimator:
\begin{equation}\label{eq:IPW_estimators_ATE}
\begin{aligned}
\textstyle
\hat \tau_\mathrm{ATE}^\mathrm{(IPW)} = \frac{1}{n_1}\sum_{A_i=1} Y_i- \frac1{n_1} \sum_{A_i=0}\frac{\hat\pi(X_i)}{1-\hat \pi(X_i)} Y_i \,,
\end{aligned}
\end{equation}
where $\hat\pi$ is an estimate of the \emph{propensity score} $\pi$ defined in \Cref{prop:0_run_identification}.
The estimator \eqref{eq:IPW_estimators_ATE} belongs to the family of IPW debiasing methods widely used in causal inference with observational data \citep{rosenbaum1983central}.\footnote{Note that $\cP_T=\cP_{X_i|A_i=1}$, so that the target ATE on $\cP_T$ is in fact the ATT (ATE on the treated). This explains why our IPW estimator in \Cref{eq:IPW_estimators_ATE} differs from the classical IPW estimator.}
The propensity score estimator $\hat\pi$ can be obtained by training any probabilistic binary classifier $\cX \to [0,1]$ to distinguish members from non-members, yielding an estimate of the probability that any $x \in \cX$ is a member.
Member and non-member data can be split into two independent parts, one used in \Cref{alg:zero} and one used to learn the classifier $\hat\pi$. Cross-fitting, which alternates the roles of these splits and averages the results, can be used to leverage all the data efficiently.
\Cref{sec:experiments} shows how $\hat \pi$ is estimated in practice: the overlap assumption is made on the underlying data representations (e.g.,  text or image embeddings), and propensity scores are functions of these representations.
 
A well-trained propensity score model captures the distributional differences between the conditional feature distributions $X_i \mid A_i = 1$ and $X_i \mid A_i = 0$, assigning higher scores to regions of the feature space that are more characteristic of members than non-members.
Consistency of $\hat \pi$ is measured as $\Delta_{\hat \pi} = \E_{\cP_0}\left[\left|\frac{\pi(X_i)}{1-\pi(X_i)}-\frac{\hat\pi(X_i)}{1-\hat\pi(X_i)}\right|\right]$:
$\hat \pi$ is $\ell^1$-consistent if $\Delta_{\hat \pi}\to0$.
Next assumption is necessary to ensure that each data point $x\in\cD_T$ has a positive (and lower-bounded, as opposed to \Cref{hyp:distribution_shift}) probability of being observed in the non-member set.
\Cref{thm:consistency_0_run} shows that if propensity scores are learned well-enough, the ATE for the negative loss-based attack is well estimated by $\hat \tau_\mathrm{ATE}^\mathrm{(IPW)}$.

\begin{assumption}\label{hyp:overlap}
    There exists $\eta>0$ such that $\pi(x) \leq 1-\eta$ and $\hat\pi(x) \leq 1-\eta$ for all $x\in \supp(\cP_T)$.
\end{assumption}

\begin{theorem}\label{thm:consistency_0_run}
    Assume that $(X_i,A_i,Y_i)_{i\in[n]}$ are obtained using $\Cref{alg:zero}$ with $Y_i=-\cL(\theta,X_i)$, that \Cref{hyp:overlap,hyp:distribution_shift} hold, and that $\cA$ is $\alpha$-error stable and $\beta$-uniform training stable. Assume outcomes $Y_i$ are in $[0,1]$ almost surely.
    Furthermore, assume that $\hat\pi$ is independent from $(X_i,A_i,Y_i)_{i\in[n]}$.
Then, with probability $1-Ce^{-t}$:
    \begin{equation*}
        |\hat \tau_\mathrm{ATE}^\mathrm{(IPW)}-\tau_\mathrm{ATE}| = \textstyle\cO\big( \Delta_{\hat\pi} + \frac{1}{\eta}\sqrt{t/n} + \sqrt{n t\alpha^2}+ \sqrt{n t\beta^2}\big)\,.
    \end{equation*} 
    \end{theorem}

\Cref{thm:consistency_0_run} has the same overall structure as the one-run result (\Cref{thm:consistency_1_run}), up to two additional terms: the additive bias $\Delta_{\hat\pi}$ of the propensity score estimator, and the factor $1/\eta$. The presence of $1/\eta$ is expected: when $\eta$ is close to zero, inverse-propensity weighting (IPW) estimators are known to become unstable.
Alternative causal inference estimators can be used to mitigate this instability. 
Letting $\hat\mu_0$ be an outcome regression model that estimates conditional MIA response $\mu_0:x\in\cX\mapsto \esp{Y_i(0)|X_i=x}$, the \emph{G-Formula} \citep{wang2017g} and \emph{Augmented Inverse Probability Weighting} (AIPW) estimators write as \citep{shu2018improved}:
\begin{equation}\label{eq:G_formula_AIPW_ATE} 
\textstyle
    \hat\tau_\mathrm{ATE}^\mathrm{(G)} = \frac{1}{n_1}\sum_{A_i=1}Y_i-\hat\mu_0(X_i)\,,\qquad     \hat\tau_\mathrm{ATE}^\mathrm{(AIPW)} \!\!=\!\hat\tau_\mathrm{ATE}^\mathrm{(G)}\!-\!\frac{1}{n_1}\sum_{i=1}^n\!\frac{(1-A_i)\hat\pi(X_i)}{1-\hat\pi(X_i)}\!\left( Y_i\!-\!\hat\mu_0(X_i) \right).
\end{equation}
The outcome regression model $\hat\mu_0$ is typically learned by fitting a regression of $Y_i$ on $X_i$ using non-member data points (i.e., those with $A_i = 0$). 
The AIPW estimator is \emph{doubly robust}, meaning it is consistent if either the propensity score model or the outcome regression is correctly specified.
Although they are more robust to extreme propensity scores, G-Formula and AIPW still require \Cref{hyp:distribution_shift,hyp:overlap} to hold.
\Cref{app:g-form_AIPW} generalizes these two estimators to other causal metrics. 
\begin{tcolorbox}[colback=gray!5!white,colframe=gray!85!white,top=2pt,bottom=2pt,title= \textbf{Insight \# 3}]
Propensity-score based correction enables unbiased evaluation of MIAs in the zero-run regime via simple binary classifiers. This yields a practical, scalable way to correct distribution-shift bias without requiring model retraining or randomized data inclusion.
\end{tcolorbox}
\looseness=-1

\section{Experiments}
\label{sec:experiments}

We illustrate how our causal framework debiases MIA evaluation across four settings. In \Cref{fig:ROC_curves_merged,fig:biden_trump}, Multi-Run, One-Run, and Zero-Run correspond to \Cref{alg:n,alg:one,alg:zero}, respectively. \textit{Raw} denotes uncorrected estimates, while \textit{Corrected} applies propensity-score adjustment in the Zero-Run regime.
Finally, \textit{Zero-Run (IID)} serves as an idealized baseline used for comparison and sanity checking. In this setting, non-members are obtained as a randomized test split from the same dataset as the members, so no correction is needed.
Additional experimental details, results and visualizations are provided in \Cref{app:exp_analysis}.

\begin{figure*}
    \includegraphics[width=\linewidth]{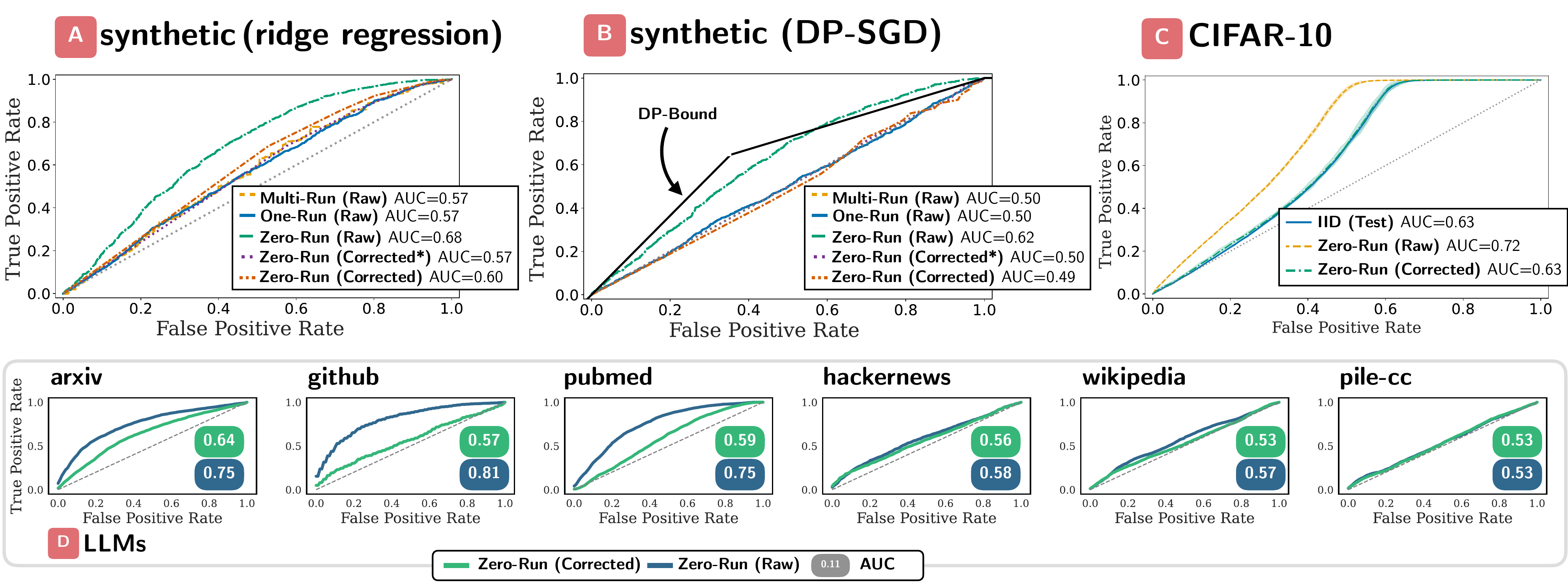}
    \caption{ROC curves and AUCs for linear regression on synthetic data (\textbf{A} and \textbf{B}), Resnet-16 on CIFAR-10 data (\textbf{C}), and \texttt{Pythia-12b} on six different text corpora (\textbf{D}).}
    \label{fig:ROC_curves_merged}
    \vspace{-15pt}
\end{figure*}

%

\textbf{Synthetic data.}
We consider linear regression on Gaussian samples, using overparameterized ridge regression and underparameterized DP-SGD, both satisfying error stability and uniform training stability.
We evaluate a loss-based MIA under three regimes: \emph{multi-run}, \emph{one-run}, and \emph{zero-run}, with non-members drawn from a shifted Gaussian. 
We consider both oracle propensity scores (denoted as \textit{Corrected$^*$}), and estimated ones with logistic regression.
ROC curves and AUCs are reported in \Cref{fig:ROC_curves_merged} (A-B). Multi-run and one-run evaluations nearly coincide as interference is negligible, whereas naive zero-run overestimates MIA performance under distribution shift.
Propensity-score correction restores agreement with non-shifted baselines.
For DP-SGD, the naive zero-run evaluation violates the ROC upper bound implied by $(\eps,\delta)$-DP \citep{Dong22GDP}, while the corrected versions respects it, highlighting the importance of debiasing.\looseness=-1

\textbf{ResNet on CIFAR-10.}
We perform zero-run evaluation on CIFAR-10 \cite{krizhevsky2009learning} using a ResNet-16. Members are sampled from a mixture of clean and noisy images (90\%/10\%) and non-members from the reverse (10\%/90\%), inducing strong distribution shift while satisfying overlap. Propensity scores are estimated via a fine-tuned ImageNet-pretrained ResNet classifier. \Cref{fig:ROC_curves_merged} (C) reports ROC curves and AUCs for a loss-based MIA.  
Naive zero-run evaluation inflates attack performance (AUC $\approx 0.72$) due to distribution shift, misleadingly suggesting high memorization. Propensity-score correction recovers the true causal signal (AUC $\approx 0.63$), closely matching an IID baseline where non-members are drawn from the same mixture as members (AUC $\approx 0.65$), confirming effective debiasing.\looseness=-1

\textbf{Pretrained LLM.}
We perform a zero-run evaluation on \texttt{Pythia-12b} \citep{biderman2023pythia}, an LLM for which membership information is available for a subset of data points. We consider $6$ datasets where MIA evaluation biases has been recently highlighted \citep{meeus2024sok} due to distribution shift induced by deduplication \citep{duan2024membership}. For each corpus, we split data into train/test sets stratified by membership, vectorize training data using a \texttt{BagOfWords} encoder \citep{meeus2024sok}, and fit a \texttt{RandomForest} classifier to estimate propensity scores. We consider the standard MIA based on perplexity.
Results are shown in \Cref{fig:ROC_curves_merged} (D).
\begin{wrapfigure}[15]{r}{0.4\columnwidth}
\vspace{-5pt}
    \centering
\includegraphics[width=.9\linewidth]{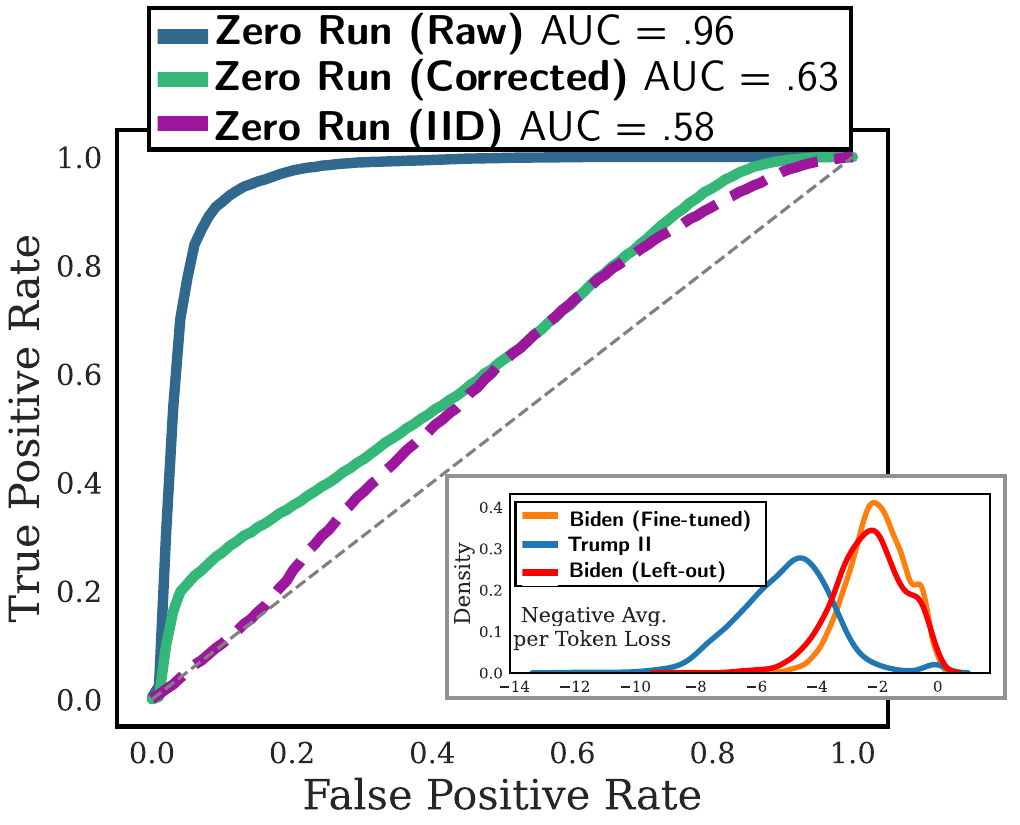}    
   \caption{Fine-tuned \texttt{Pythia-70m} model: raw, corrected and pseudo-ground truth (IID) ROC curves. 
} 
\label{fig:biden_trump}
\end{wrapfigure}
When distribution shift is large (e.g., \texttt{arxiv}, \texttt{github}, \texttt{pubmed\_central}), members become easily distinguishable from non-members and naive evaluation substantially inflates ROC curves. 
In contrast, causal correction largely removes the signal, indicating little true memorization, consistent with prior work on MIAs for pretrained LLMs \cite{duan2024membership}.
For \texttt{hackernews}, \texttt{wikipedia}, and \texttt{pile\_cc}, shifts are weaker and the correction is correspondingly modest.\looseness=-1

\textbf{Fine-tuned LLM.} To construct an idealized IID reference setting for MIA evaluation on an LLM, we fine-tune \texttt{Pythia-70m} on a scraped dataset of White House press statement titles under the Biden (2021-2025) and Trump II (from 2025) administrations, denoted by $B$ and $T$ respectively, between which strong linguistic distribution shifts have been documented \citep{korner2022linguistic}).\footnote{We manually verified that none of these texts appear in the pretraining data of \texttt{Pythia-70m}.} We further split $B$ into two sets $(B_1,B_2)$ and fine-tune the model on $B_1$ using \texttt{LoRA} \citep{hu2022lora}. This controlled construction enables an IID evaluation baseline: since $B_2$ is drawn from the same distribution as $B_1$, MIA performance on $(B_1, B_2)$ serves as a pseudo-ground truth reference, in contrast to $(B_1, T)$ which exhibits distribution shift. We evaluate a loss-based MIA using the negative average per-token loss on both $(B_1, B_2)$ and $(B_1, T)$, with and without propensity-score correction (zero-run). Results are shown in \Cref{fig:biden_trump}.
We find that naive MIA on $(B_1, T)$ is severely inflated by distribution shift (AUC $> 0.95$), while propensity-score correction closely matches the IID reference $(B_1, B_2)$ (AUC $\approx 0.6$), supporting the effectiveness of our debiased estimator.

\begin{tcolorbox}[colback=gray!5!white,colframe=gray!85!white,top=2pt,bottom=2pt,title= \textbf{Insight \# 4}]
   Our approach offers a practical solution to a key problem in MIA for LLMs: correcting evaluation biases induced by distribution shift, which systematically inflate perceived memorization \citep{meeus2024sok,duan2024membership}. Our framework yields faithful estimates of true memorization in language models.\looseness=-1
\end{tcolorbox}

\section{Related Work}

An extended discussion of related work is provided in Appendix~\ref{app:related_works}.

\textbf{Limits of the one-run regime.} Recent work has studied the fundamental limits of the one-run regime for auditing differential privacy (DP) guarantees \citep{Keinan2025,Xiang2025}. Their objectives and conceptual frameworks differ from ours: they derive DP lower bounds, whereas we aim to obtain unbiased MIA evaluation.\looseness=-1

\textbf{Formalizing and addressing MIA bias.}
MIAs are traditionally framed as randomized privacy games \citep{shokri_enhanced,privacy_games}. However, unlike our causal framework, these games fail to capture the interference and distribution shifts characteristic of the one-run and zero-run regimes \citep{meeus2024sok}. Focusing on distribution shift in MIA evaluation for LLMs, prior work \citep{meeus2024sok,eichler_nob-mias_2025} suggest potential solutions, but these require intervention prior or during training (e.g., randomized train/test splits or data injection) or rely on heuristics. \citet{kazmi2024panoramia} propose using synthetic data points, but this approach is costly, data-hungry, and difficult to ground theoretically. In constract, we provide principled, causally valid estimators that rely only on simple binary classifiers or regressors.

\textbf{Causality and MIAs.} We mention two prior works at the intersection of MIAs and causality, though orthogonal to our focus. \citet{tople2020alleviating}
show that causal learning can effectively reduce MIA success---a topic orthogonal to our work.
\citet{10.1145/3548606.3560694} use causal reasoning to investigate whether MIA performance arises solely from imperfect generalization, as suggested by prior work \citep{yeom2018privacy,sablayrolles2019white}.

\textbf{Causal inference under interference.} Contrarily to previous work which has focused on known and sparse interference patterns \citep{sobel2006randomized,hudgens2008toward,manski2013identification,papadogeorgou2019causal,bhattacharya2020causal,clark2021approach,li2022random}, interference in our setting is dense, unobserved, and arises as a random function of the covariates $X_i$ through the training algorithm $\mathcal{A}$, which acts as a mediator (Figure~\ref{fig:causal_graph}). Consequently, we
impose assumptions on the mediator itself, leveraging learning-theoretic bounds on the influence of individual training points on the model's behavior. Our estimands are thus not conditioned on a interference graph.
Together, our mediator-induced interference model, the unconditional estimands, and the mild assumptions we impose distinguish our setting from existing work and make Theorem~\ref{thm:consistency_1_run} independently interesting within the interference literature.\looseness=-1

\section{Conclusion}

We introduced a causal framework for MIA evaluation, providing principled, statistically valid estimators that account for interference and distribution shifts in one-run and zero-run regimes. Our approach unifies the evaluation of MIAs, clarifies the root causes of bias, and enables practical post-hoc assessment without relying on model provider trust.

While we focus on population‑level metrics—evaluating on-average memorization over a target population—it is also crucial to understand subgroup disparities in memorization \citep{carlini2022membership,DBLP:journals/popets/KulynychYCVT22}. In causal terms, this amounts to estimating heterogeneous treatment effects \citep{Wager03072018}. Our framework provides a foundation for principled subgroup‑level memorization estimates, even in challenging one-run and zero-run regimes.
Moreover, this causal perspective may extend beyond MIAs to other privacy attacks, such as attribute inference \citep{10.1145/2810103.2813677}, where existing evaluation biases \citep{10.1145/3548606.3560663} could be mitigated by counterfactual reasoning.


\section*{Acknowledgements}

Parts of this work were conducted while LB was a member of PreMeDICaL team, Inria, Idesp, Inserm, Université de Montpellier.
LB’s work at EPFL is supported by an EPFL AI Center Postdoctoral Fellowship. Parts of the experiments of this paper were run on the RCP Cluster at EPFL. 

This work is partially supported by grant ANR-20-CE23-0015 (Project PRIDE), ANR 22-PECY-0002 IPOP (Interdisciplinary Project on Privacy) project of the Cybersecurity PEPR, and ANR-22-PESN-0014 project of the Digital Health PEPR under the France 2030 program. This work was performed
using HPC resources from GENCI–IDRIS (Grant 2023-AD011014018R3).

\bibliography{references}
\bibliographystyle{plainnat}

\newpage
\appendix
\onecolumn

\section{Extended Related Work}
\label{app:related_works}
\textbf{Limits of the one-run regime.} Recent work has studied the fundamental limits of the one-run regime for auditing differential privacy (DP) guarantees \citep{Keinan2025,Xiang2025}. These studies, like ours, trace the key limitation to interference between data points. Beyond this shared root cause, the objectives and conceptual frameworks differ: the prior work focuses on establishing DP lower bounds using information-theoretic perspectives such as bit transmission or hypothesis testing, whereas we focus on evaluating MIAs, adopt a causal framing that explicitly captures how interference can invalidate counterfactual reasoning, and propose solutions to mitigate these limitations.

\textbf{Formalizing and addressing MIA bias.}
MIAs are traditionally framed as randomized privacy games \citep{shokri_enhanced,privacy_games}. However, unlike our causal framework, these games fail to capture the interference and distribution shifts characteristic of the one-run and zero-run regimes \citep{meeus2024sok}. Focusing on distribution shift in MIA evaluation for LLMs, \citet{meeus2024sok} suggest potential solutions, but these require intervention priors or during training (e.g., randomized train/test splits or data injection) or rely on heuristics. \citet{kazmi2024panoramia} propose using synthetic data points from a generative model trained on the member data as non-members, but this approach is costly, data-hungry, and difficult to evaluate. While they provide theoretical conditions for obtaining meaningful DP lower bounds, their method offers no guarantees for MIA performance estimates. \citet{eichler_nob-mias_2025} develops techniques to mitigate distribution shift in post-hoc text data, yet offer no formal guarantees for the resulting MIA estimates. In contrast, we provide principled, causally valid estimators that rely only on simple binary classifiers or regressors.

\citet{zhang2024membership} argue that several alternative techniques can be used to demonstrate to a third party that a model was trained on their data when the training set is unknown and retraining is impractical, including data watermarking \citep{sablayrolles_radioactive_2020,kirchenbauer_watermark_2024}, random canary insertion \citep{carlini2019secret}, and data reconstruction attacks \citep{carlini2021extracting,nasr2023scalable}. The first two approaches require proactive intervention and trust in the model provider, while data reconstruction is fundamentally more difficult than membership inference. Developing statistically valid post-hoc MIA evaluations that do not rely on model provider cooperation or trust---the focus of our work---thus remains an important and practically relevant problem.\looseness=-1

\textbf{Causality and MIAs.} We mention two prior works at the intersection of MIAs and causality, though orthogonal to our focus. \citet{tople2020alleviating}
show that causal learning can effectively reduce MIA success: unlike predictive models that may exploit spurious correlations, causal models rely on stable, invariant relationships between input and features outcomes.
\citet{10.1145/3548606.3560694} use causal reasoning to investigate whether MIA performance arises solely from imperfect generalization, as suggested by prior work \citep{yeom2018privacy,sablayrolles2019white}. They find that while overfitting contributes to MIA success, it is in fact driven by a complex interplay of multiple factors.

\textbf{Causal inference under interference.}
Seminal work formalized causal inference with interference via exposure maps \citep{sobel2006randomized,hudgens2008toward,manski2013identification}, which describe how each unit's outcome depends on other units' treatment assignments. These mappings are often represented by a graph, whose edges encode potential interference. Subsequent work has largely focused on fixed and sparse exposure maps \citep{bhattacharya2020causal,papadogeorgou2019causal,ogburn2024causal}, with extensions to random \citep{clark2021approach,li2022random}, weighted \citep{forastiere2024causal}, and dense \citep{miles2019causal} settings. In these frameworks, estimands are typically direct or indirect effects defined conditional on an observed exposure map.
In contrast, interference in our setting is dense, unobserved, and arises as a random function of the covariates $X_i$ through the training algorithm $\mathcal{A}$, which acts as a mediator (Figure~\ref{fig:causal_graph}). Rather than observing or modeling the exposure map directly, we impose assumptions on the mediator itself, leveraging learning-theoretic bounds on the influence of individual training points on the model's behavior at other points. As a result, our estimands are not conditioned on a particular exposure map: $\tau_{\mathrm{ATE}}$ averages over realizations of the mediator $\mathcal{A}$, and thus over the induced interference patterns.
Together, our mediator-induced interference model, the unconditional estimands, and the mild assumptions we impose distinguish our setting from existing work and make Theorem~\ref{thm:consistency_1_run} independently interesting within the interference literature.

\section{Additional Causal Inference Details}

\subsection{Causal Inference Assumptions}
\label{app:causal_assumptions}

We here explain the different regimes used in causal inference to estimate treatment effects. In particular, we explicit the underlying assumptions that are made in these regimes.

\paragraph{Randomized Control Trial (RCT).}
In RCTs, treatment is randomized and there are no interferences. \Cref{hyp:standard_causal_assumptions} is thus assumed to hold with \textit{2. (a)} for assignment mechanism. The overlap assumption holds, since $\proba{A_i=1|X_i}=\proba{A_i=1}$ does not depend on the covariates. The SUTVA assumption means that there are no interferences, and is usually made for RCTs.

\paragraph{Observational Studies.}
RCTs are usually expensive and have restrictive eligibility criteria, hindering both their statistical power (smaller samples) and their population.
As a consequence, the promise of observational data through the use of hospital electronic health records has emerged \citep{rosenbaum1983central}.
However this comes at the cost of non-randomized treatment assignments, and estimating treatment effects on observational data is infeasible without additional structural assumption.
As a consequence, if there are no interferences or if they are neglected, \Cref{hyp:standard_causal_assumptions} is assumed to hold with \textit{2. (b)} for the assignment mechanism: the observed covariates $X_i$ are rich enough, so that there are no unobserved confounders.
The overlap assumption is assumed to hold here, and does not hold by design as in RCTs: given covariates $X_i$, the probability of patient $i$ of being treated is bounded away from 0 and 1.

\paragraph{RCT with Interferences.}
For vaccine trials for instance, the SUTVA assumption can no longer be made, due to herd immunity.
As a consequence, it is replaced using the interference formalism introduced in \Cref{sec:background_causal} as follows.

\begin{assumption}[Assumptions under interference]\label{hyp:RCT_interferences}
\leavevmode
\begin{enumerate}
    \item \textit{(Interference)} $Y_i = Y_i(\bar A)$;
    \item \textit{(Generalized Unconfoundedness)} $\bar A \indep \{\bar Y(\bar a)\}_{\bar a}$;
    \item \textit{(Generalized Overlap)} For all $i\neq j$,
    \[
    0<\proba{A_i=a\mid \bar A_{-i}}<1.
    \]
\end{enumerate}
\end{assumption}

\paragraph{Observational Studies with Interferences.}
Combining both difficulties (confounders and interferences), such as vaccine observational data, requires making the following assumption.

\begin{assumption}[Interference and confounders]\label{hyp:observational_interferences}
\leavevmode
\begin{enumerate}
    \item \textit{(Interference)} $Y_i = Y_i(\bar A)$;
    \item \textit{(Conditional Generalized Unconfoundedness)} 
    $\bar A \indep \{\bar Y(\bar a)\}_{\bar a} \mid \bar X$;
    \item \textit{(Conditional Generalized Overlap)} For all $i\neq j$,
    \[
    0<\proba{A_i=a\mid \bar X,\bar A_{-i}}<1 \ \text{a.s.}
    \]
\end{enumerate}
\end{assumption}

\subsection{Discussion on the Causal AUC}
\label{app:causal_AUC_discussion}

The classical AUC compares the scores of outcomes for treated ($A_i=1$) and control ($A_i=0$) datapoints, and can be expressed as:
\begin{equation*}
    \proba{Y_i\geq Y_j|A_i=1,A_j=0}\,.
\end{equation*}
It is however not straightforward to write the causal AUC, defined as the area under the causal ROC curve, as a comparison estimand.
In causal inference, the most common quantity for comparing potential outcomes is the probability of necessity and sufficiency (PNS) \citep{tian2000probabilities}, defined as $\mathrm{PNS} = \proba{Y_i(1) > Y_i(0)}$.
This represents the probability that a given datapoint has a higher score if treated compared to not being treated. While PNS is useful for assessing treatment effectiveness, it is unidentifiable in practice.
Instead, the Mann-Whitney-Wilcoxon test statistics \citep{Wilcoxon1945, Mann1947}, along with other pairwise comparison methods like the Win Ratio or Net Benefit \citep{Pocock2011winratio, Buyse2010, Mao2017, even2025winratio}, estimate an identifiable quantity known as the \textit{Win Proportion} (WP), given by:
\begin{equation*}
\mathrm{WP} \eqdef \proba{Y_i(1) > Y_j(0) \mid A_i = A_j = 1},.
\end{equation*}
This estimand is a causal interpretation of the AUC, replacing the unknown counterfactual with that of another datapoint, making it identifiable under standard causal assumptions.
However, the WP is not well-defined in the presence of interference, which makes it unidentifiable. In the one-run and zero-run settings, the joint distribution of $(Y_i(1), Y_j(0))$ becomes unidentifiable because changes in $A_i$ or $A_j$ (due to datapoint insertions) affect the potential outcomes $Y_i(1)$ and $Y_j(0)$. Therefore, a Causal Comparison AUC can be defined as follows using causal inference with interference notations.

\begin{definition}[Causal Comparison AUC]\label{def:causal_comparison_AUC}
    For $i\ne j$, and $a_i,a_j\in\set{0,1}$, define $Y_{i,j}(a_i,a_j)$ as $Y_{i,j}(a_i,a_j)= $
    \begin{align*}
            Y_i(A_1,\ldots,A_{i-1},a_i,A_{i+1},\ldots,A_{j-1},a_j,A_{j+1},\ldots,A_n).
\end{align*}
We then define the causal AUC as:
\begin{equation*}
    \tilde \tau_\mathrm{AUC} \eqdef \proba{Y_{i,j}(1,0)>Y_{j,i}(0,1)\,|\,A_i=A_j=1}\,.
\end{equation*}
\end{definition}

Importantly, note that if SUTVA holds (no interferences, $Y_i(\bar a)=Y_i(a_i)$ for any $\bar a \in\set{0,1}^n$) we have $Y_{i,j}(a_i,a_j)=Y_i(a_i)$ and $\tau_\mathrm{AUC}=\mathrm{WP}$. \Cref{def:causal_comparison_AUC} is thus a strict generalization of the WP estimand to the interference setting.
Under interferences, $\tilde\tau_\mathrm{AUC}$ is however in general not identifiable, even under \Cref{hyp:RCT_interferences} or \Cref{hyp:observational_interferences}, that would instead need to be replaced by the following (interference observational setting).

\begin{assumption}[Alternative assumption]\label{hyp:observational_interferences_2}
\leavevmode
\begin{enumerate}
    \item \textit{(Interference)} $Y_i = Y_i(\bar A)$;
    \item \textit{(Conditional Generalized Unconfoundedness)} 
    $\bar A \indep \{\bar Y(\bar a)\}_{\bar a} \mid \bar X$;
    \item \textit{(Conditional Generalized Overlap)} For all $i\neq j$,
    \[
    0<\proba{A_i=a,A_j=a'\mid \bar X,\bar A_{-(i,j)}}<1 \ \text{a.s.}
    \]
\end{enumerate}
\end{assumption}

\subsection{Identifiability Proofs}

The proofs of the ATE parts of \Cref{prop:n_run_identification,prop:1_run_identification,prop:0_run_identification} are all implied by the following.

\begin{proposition}\label{prop:ATE_identification_app}
    Assume that \Cref{hyp:observational_interferences} holds. Then:
        \begin{equation}\label{eq:ATE_identification_app}
    \textstyle
        \tau_\mathrm{ATE} = \frac{\esp{A_iY_i}}{\proba{A_i=1}} - \frac{\E\big[ \tfrac{\pi(X_i)(1-A_i)}{1-\pi(X_i)} Y_i\big]}{\proba{A_i=0}}\,.
    \end{equation}
\end{proposition}

\Cref{prop:ATE_identification_app} implies \Cref{prop:0_run_identification} since under \Cref{hyp:distribution_shift} we have that \Cref{hyp:observational_interferences} holds in the Zero-run setting. 
Then, in the One-run and Multi-run settings, \Cref{hyp:RCT_interferences,hyp:standard_causal_assumptions} respectively hold, which implies \Cref{hyp:observational_interferences} and thus \Cref{prop:ATE_identification_app}. Then specifying \Cref{prop:ATE_identification_app} in the One-run setting for which $\pi$ is constant equal to 0.5, we get the ATE parts of \Cref{prop:1_run_identification,prop:n_run_identification}.
Identification for the other causal metrics follow from similar arguments.

\begin{proof}[Proof of \Cref{prop:ATE_identification_app}]
Under the assumptions of either of \Cref{prop:n_run_identification,prop:1_run_identification,prop:0_run_identification}, we have that \Cref{hyp:observational_interferences} holds.
First,
\begin{align*}
    \E_{\cP_T}[Y_i(1)] &= \esp{Y_i(1)|A_i=1}\\
    &= \esp{Y_i | A_i=1}\quad \text{(Interference assumption in \Cref{hyp:observational_interferences})}\\
    &=\esp{A_iY_i|A_i=1}\,.
\end{align*}
Then,
\begin{align*}
    \esp{A_iY_i}&=\esp{A_iY_i|A_i=1}\proba{A_i=1}+\esp{A_iY_i|A_i=0}\proba{A_i=0}\\
    \esp{A_iY_i|A_i=1}\proba{A_i=1}+0\,,
\end{align*}
leading to $\E_{\cP_T}[Y_i(1)] = \frac{\esp{A_iY_i}}{\proba{A_i=1}}$.
For the other term,
\begin{align*}
    \E_{\cP_T}[Y_i(0)] &= \esp{Y_i(0)|X_i\sim\cP_T}\\
    & = \esp{\esp{Y_i(0)|\bar X} |X_i\sim\cP_T}\\
    & = \esp{\esp{Y_i(0)|\bar X,\bar A_{-i},A_i=0} |X_i\sim\cP_T}\quad \text{(Conditional Generalized Unconfoundedness)}\\
    & = \esp{\esp{Y_i|\bar X,\bar A_{-i},A_i=0} |X_i\sim\cP_T}\quad \text{(Interference assumption in \Cref{hyp:observational_interferences})}\\
    & = \esp{\esp{Y_i|X_i,A_i=0} |X_i\sim\cP_T}\\
    & = \esp{\esp{(1-A_i)Y_i|X_i,A_i=0} |X_i\sim\cP_T}\,.
\end{align*}
Now,
\begin{align*}
    \esp{(1-A_i)Y_i|X_i} &= \esp{(1-A_i)Y_i|X_i,A_i=0}\proba{(1-A_i)A_i=0|X_i} \\
    &\quad+ \esp{(1-A_i)Y_i|X_i,A_i=1}\proba{A_i=1|X_i}\\
    &= \esp{(1-A_i)Y_i|X_i,A_i=0}(1-\pi(X_i))+ 0\,.
\end{align*}
This leads to:
\begin{align*}
    \E_{\cP_T}[Y_i(0)] &=\esp{\esp{(1-A_i)Y_i|X_i}(1-\pi(X_i))^{-1} |X_i\sim\cP_T}\\
    &=\esp{g(X_i)|X_i\sim\cP_T}\,.
\end{align*}
with $g(x)=\esp{(1-A_i)Y_i|X_i=x}(1-\pi(x))^{-1} $.
Now, notice that the density of $\cP_T$ with respect to $\cP_{X_i}$ is proportional to $\pi$: $\frac{\dd\cP_T(x)}{\dd\cP_{X_i}}(x)=\frac{\pi((x)}{\proba{A_i=1}}$.
Thus, $\esp{g(X_i)|X_i\sim\cP_T}=\esp{\frac{\pi(X_i)}{\proba{A_i=1}}g(X_i)|X_i\sim\cP_{X_i}}$, concluding the proof.
\end{proof}

\section{Results in the underparametrized regime}
\label{app:underparam}

The next theorem considers the underparameterized setting and shows that MIAs based on the loss cannot recover any signal ($\tau_{\rm ATE} \approx 0$) when the number of samples becomes large in front of the dimension of the model (i.e., $\frac{D}{n}\to0$).


\begin{theorem}[Underparametrized Regime] \label{thm:under_1_run}
    Consider the assumptions of \Cref{thm:consistency_1_run} but without uniform training stability. Further assume that $\theta=\cA(\cD)\in\mathbb{R}^D$ has norm almost surely bounded by a constant $B$ and that $\theta \mapsto \cL(\theta,x)$ is $L$-Lipschitz for all $x\in \cX$. 
    Then:
\begin{equation*}
     \textstyle|\tau_\mathrm{ATE}| = \cO\Big(   \sqrt{n \log(n) \alpha^2} + \sqrt{\frac{D \log (B L n)}{n}}\Big)\,.
\end{equation*}   
\end{theorem}

\section{Estimators for and generalization to all Causal MIA Evaluation Metrics}
\label{app:estimators_causal_metrics}

\subsection{Multi-run setting}

Estimators of the other causal metrics beyond the ATE in the multi-run setting write as:
\begin{equation}\label{eq:estimators_1_run}
\begin{aligned}
    &\hat \tau_\mathrm{ATE} = \frac{1}{n_1}\sum_{i:A_i=1}Y_i-\frac{1}{n_0}\sum_{i:A_i=0}Y_i \,, \\
    &\hat \tau_\mathrm{TPR}(t) = \frac{1}{n_1}\sum_{i:A_i=1}\one_\set{Y_i\geq t}\\
    &\hat \tau_\mathrm{FPR}(t) =\frac{1}{n_0}\sum_{i:A_i=0}\one_\set{Y_i\geq t} \, \\
    &\hat \tau_\mathrm{AUC} = \frac{1}{n_1n_0}\sum_{(i,j):(A_i,A_j)=(1,0)}\one_\set{Y_i> Y_j} \,, \\
    &\hat \tau_\mathrm{TPR@FPR}(\alpha) =  \frac{1}{n_1}\sum_{i:A_i=1}\one_\set{Y_i\geq \hat t_\alpha}\,,
\end{aligned}
\end{equation}
where $\hat t_\alpha$ is an estimator of the $1-\alpha$ quantile of $\set{Y_i:A_i=0}$, of respectively the ATE (causal membership advantage), the causal TPR and FPR at threshold $t\in\R$, the causal AUC and the causal TPR at fixed FPR $\alpha\in (0,1)$.

\begin{proposition}\label{prop:mia_evaluation_n_run_2}
    Assume that $(X_i,A_i,Y_i)_{i\in[n]}$ is obtained in the Multi Run setting (\Cref{alg:n}).
    Then, $\hat\tau_\mathrm{ATE}$, $\hat\tau_\mathrm{ATE}(t)$ for $t\in\R$ and $\hat\tau_\mathrm{AUC}$ are consistent estimators of $\tau_\mathrm{ATE}$, $\tau_\mathrm{ATE}(t)$ and $\tau_\mathrm{AUC}$. If $t\mapsto \mathrm{FPR}(t)$ is continuous, $\hat \tau_\mathrm{TPR@FPR}(\alpha)$ is a consistent estimator of $ \tau_\mathrm{TPR@FPR}(\alpha)$ for $\alpha\in[0,1]$.
\end{proposition}

Moving beyond consistency, non-asymptotic rates can be obtained directly from the difference in means estimator properties \citep[e.g.,]{wager2024causal}. For $\hat\tau_\mathrm{TPR}(t)$, with probability $1-4e^{-\lambda}$,
\begin{equation*}
    |\hat\tau_\mathrm{TPR}(t)-\tau_\mathrm{TPR}(t)| \leq \sqrt{\frac{\lambda}{2n_1}} + \sqrt{\frac{\lambda}{2n_0}}\,,
\end{equation*}
and similarly for $\tau_\mathrm{TPR}$,
while for the causal membership advantage (ATE) if outcomes $Y_i$ are bounded by a constant $B$:
\begin{equation*}
    |\hat\tau_\mathrm{ATE}-\tau_\mathrm{ATE}| \leq B\sqrt{\frac{2\lambda}{n_1}} + B\sqrt{\frac{2\lambda}{n_0}}\,.
\end{equation*}
Non-asymptotic results for the causal AUC can be obtained via $U$-statistics \citep{pitcan2017note}.

\subsection{One-run setting}

In the one-run setting, the estimators of the different causal evaluation metrics are the same as in the multi-run (see \Cref{eq:estimators_1_run}), at the exception of the causal AUC estimator, which is now the area under the estimated causal ROC curve $\set{(\hat \tau_\mathrm{FPR}(t),\hat \tau_\mathrm{TPR}(t))}$.

\subsection{Zero-run setting}

In the zero-run setting, the estimators of the different causal evaluation metrics are now weighted with propensity scores as follows:
\begin{equation}\label{eq:IPW_estimators}
\begin{aligned}
&\hat \tau_\mathrm{ATE}^\mathrm{(IPW)} = \frac{1}{n_1}\sum_{A_i=1} Y_i- \frac1{n_0} \sum_{A_i=0}\frac{\hat\pi(X_i)}{1-\hat \pi(X_i)} Y_i \,, \\
    &\hat \tau_\mathrm{TPR}^\mathrm{(IPW)}(t) = \frac{1}{n_1}\sum_{A_i=1} \one_{\{Y_i \geq t\}}\\
   &\hat \tau_\mathrm{FPR}^\mathrm{(IPW)}(t) =  \frac1{n_1} \sum_{A_i=0}\frac{\hat\pi(X_i)}{1-\hat \pi(X_i)} \one_{\{Y_i\geq t\}} \, \\
   & \hat \tau_\mathrm{TPR@FPR}^\mathrm{(IPW)}(\alpha) =  \frac{1}{n_1}\sum_{A_i=1} \one_\set{Y_i\geq \hat t_\alpha} \,
\end{aligned}
\end{equation}
with 
\begin{equation*}
    \hat t_\alpha \in \argmax\left\{t~:~\frac1{n_0}\sum_{A_i=1} \frac{\hat\pi(X_i)}{1-\hat \pi(X_i)} \one_{\{Y_i\geq t\}} \geq 1-\alpha \right\}\,,
\end{equation*}
where $\hat\pi:\cX\to[0,1]$ is an estimator of the propensity score function $\pi$, learned on independent data.
The estimator of the causal AUC is still the area under the estimated ROC curve.
Next proposition then generalizes the identification result established in \Cref{prop:0_run_identification} to the other causal evaluation metrics.

\begin{proposition}\label{prop:0_run_identification_formulas}
    Under \Cref{hyp:distribution_shift}, if $(X_i,A_i,Y_i)_{i\in[n]}$ are obtained according to \Cref{alg:zero}, the causal evaluation metrics are identifiable.
    Let the \textit{propensity score} $\pi$ write as:
    \begin{equation*}
        \forall x\in\cX\,,\quad \pi(x)\eqdef \proba{A_i=1|X_i=x}\,.
    \end{equation*}
    We have:
    \begin{equation*}
        \begin{aligned}
            \tau_\mathrm{ATE} &= \esp{A_iY_i}\proba{A_i=1}^{-1}\\
            &\quad- \esp{ \frac{\pi(X_i)(1-A_i)}{1-\pi(X_i)} Y_i}\proba{A_i=1}^{-1}\,,\\
            \tau_\mathrm{TPR}(t) &= \esp{A_i\one_\set{Y_i\geq t}}\proba{A_i=1}^{-1}\\
            \tau_\mathrm{FPR}(t) &= \esp{ \frac{\pi(X_i)(1-A_i)}{1-\pi(X_i)} \one_\set{Y_i\geq t}}\proba{A_i=1}^{-1}\,.
        \end{aligned}
    \end{equation*}
\end{proposition}

\section{G-Formula and AIPW beyond the ATE}
\label{app:g-form_AIPW}

Letting $\hat\mu_0$ be a regression model that estimates conditional MIA response $\mu_0:x\in\cX\mapsto \esp{Y_i(0)|X_i=x}$, the G-Formula estimator of the ATE (which is in fact an ATT) is defined in \Cref{eq:G_formula_AIPW_ATE}.
Outcome regression model $\hat\mu_0$ is typically learned on an independent dataset or using crossfitting, by fitting a model $Y_i\sim X_i$ on non-member datapoints (satisfying $A_i=0$).
We now generalize this approach to the causal TPR and FPR, in order to obtain a causal ROC curve and a causal AUC, without having to resort to propensity scores.

\paragraph{TPR.}
For the causal TPR, the G-Formula estimator $\hat\tau^{(\rm G)}_\mathrm{TPR}(t)$ is unchanged (same as $\hat\tau_\mathrm{TPR}^{(\rm IPW)}(t)$ in \Cref{eq:IPW_estimators}). The same for the AIPW estimator.

\paragraph{FPR.}
For the causal TPR, the G-Formula estimator $\hat\tau^{(\rm G)}_\mathrm{FPR}(t)$ requires to learn a family of regressors $\set{\hat\mu_{0,t},t\in\R}$ where $\hat\mu_{0,t}$ is an estimator of:
\begin{equation*}
    \mu_{0,t}:x\in\cX\mapsto \proba{Y_i(0)\geq t|X_i=x}\,.
\end{equation*}
Since $\mu_{0,t}(x)=\proba{Y_i\geq t|X_i=x,A_i=0}$, $\hat\mu_{0,t}$ can be learned by using a classifier and outputting conditional probabilities, that we fit on the non-members $\set{(X_i,\one_\set{Y_i\geq t})|A_i=0}$. This needs to be done for all $t\in\R$ or for a discretized set.
The causal FPR at threshold $t$ can the be estimated using:
\begin{equation*}
    \hat\tau_\mathrm{FPR}^{(\rm G)}(t)=\frac{1}{n_1}\sum_{A_i=1}\hat\mu_{0,t}(X_i)\,.
\end{equation*}
Another approach would be to use \textit{distributional regression}, which is typical of comparison-based estimands in causal inference \citep{even2025winratio}. An estimator $\hat\P_0(x,\cdot)$ of the probability distribution $\P_{Y|X_i=x,A_i=0}$ could be learn, where $\hat \P(x,t)$ is an estimator of $\proba{Y_i\geq t|X_i=x}$.

Finally, the AIPW estimator of the causal FPR writes as:
\begin{equation*}
    \hat\tau_\mathrm{FPR}^\mathrm{(AIPW)} (t) = \frac{1}{n_1}\sum_{i=1}^n\Big[A_i\hat\mu_{0,t}(X_i) +\frac{(1-A_i)\hat\pi(X_i)}{1-\hat\pi(X_i)}\left( \one_\set{Y_i\geq t}-\hat\mu_{0,t}(X_i) \right)\Big]
\end{equation*}

\section{Proofs in the 1-Run Randomized setting}

In all the proofs, everything is done conditionally on the initial data $\cD$, and we leave the conditioning with respect to $\cD$ implicit. For a vector $\bar a \in \{0,1\}^n$, we let $n_1(\bar a) = \#\{i~:~a_i = 1\}$ and similarly for $n_0(\bar a)$. In particular $n_1(\bar A) = n_1$. We'll use the fact that, on an event of probability at least $1-4e^{-t}$, it holds 
    $$
    \left|n_a-\frac{n}{2}\right| \leq \sqrt{\frac{tn}{2}} \quad \text{for} \quad a\in\{0,1\}. 
    $$
This is a straightforward application of Hoeffding's inequality. In particular, for $t \leq n/8$, it holds on the same event $
\frac{n}{4} \leq n_a \leq \frac{3n}{4}$ for both $a\in \{0,1\}$.

\subsection{Proof of \Cref{thm:consistency_1_run}}

\subsubsection{Proof with perfect interpolation}

\begin{proof}
    First, note that thanks to the interpolation assumption, 
    \begin{equation*}
        \frac{1}{n_1}\sum_{i:A_i=1 }Y_i=0
        =\esp{\esp{Y_i|A_i=1,X_i}}\,.
    \end{equation*}
    We thus need to work only with $\frac{1}{n_0}\sum_{i:A_i=0 }Y_i$.
    We let $a \in \{0,1\}^n$ with $n_0(\bar a) \neq 0$, and define, letting $\overline X_{\bar a} := \{X_j\}_{\{j:~a_j=1\}}$,
    $$
    \mu_{\bar a}^{(0)}(\overline X_{\bar a}) := \esp{Y_i|  \bar A= \bar a, \overline X_{\bar a}},
    $$
    for any $i$ such that $a_i=0$ (we also integrate wrt the randomness of $\cA$ in the expectation, after conditioning on $\cA$ to apply Hoeffding). 
    Thanks to Hoeffding's inequality, it holds
    \begin{align*}
        \proba{\left|\frac{1}{n_0(\bar a)}\sum_{i:a_i=0 }Y_i - \mu_{\bar a}^{(0)}(\overline X_{\bar a})\right| > \sqrt{\frac{t}{2n_0(\bar a)}} ~\middle|~ \overline A = \overline a, \overline X_{\bar a}}\leq 2e^{-t}\,.
    \end{align*}
    Now for two adjacent datasets $\overline X_{\bar a}$ and $\overline X_{\bar a}'$, it holds:  
    \begin{align*}
        &|\mu_{\bar a}^{(0)}(\overline X_{\bar a})-\mu_{\bar a}^{(0)}(\overline X'_{\bar a})|=|\E_\cA\E_{\cP_T} [\cL(\cA(\cD \cup \overline X_{\bar a}), X)]-  \E_\cA\E_{\cP_T} [\cL(\cA(\cD \cup \overline X'_{\bar a}), X)]| \leq \alpha,
    \end{align*}
   thanks to the average stability assumption (which implicitly encompasses adjacency for \textit{Replace} and \textit{Add/Remove}).
    Thus, using McDiarmid's inequality, it holds 
    \begin{equation*}
        \proba{|\mu_{\bar a}^{(0)}(\overline X_{\bar a})-\E[Y_i|\overline A=\bar a]| > \sqrt{\frac{n_1(\bar a) t \alpha^2}{2}} ~\middle|~ \overline A = \bar a}\leq 2e^{-t}.
    \end{equation*}
    Thus
    \begin{equation*}
        \proba{|\hat \tau_{\rm ATE}-\E[Y_i|\overline A = \bar a]| \leq \sqrt{\frac{2t}{n_0(\bar a)}} + \sqrt{\frac{n_1(\bar a)t\alpha^2}{2}} ~\middle | ~ \overline A = \bar a} \geq 1 - 4e^{-t}.
    \end{equation*}
    We conclude the proof by introducing 
    $$
    g(\bar a_{-i}) := \E[Y_i|A_i=0, \overline A_{-i} = a_{-i}].
    $$
    The previous bounds allow for a control of $|\hat \tau_{\rm ATE}- g(\bar A_{-i})|$. Using again the $\alpha$-stability, one find that for two adjacent vectors $\bar A_{-i}, \bar A_{-i}'$, it holds $|g(\bar A_{-i})-g(\bar A_{-i}')| \leq \alpha$. Using McDiarmid's inequality again, we conclude
    $$
    \proba{|g(\overline A_{-i})-\E[Y_i|A_i=0]| > \sqrt{\frac{(n-1) t \alpha^2}{2}}}\leq 2e^{-t}.
    $$
    Piecing all the inequalities together yields the result.

\end{proof}

\subsubsection{Proof with $\beta$-uniform training stability}

\begin{proof}
    The only difference lies in the way we handle the term 
    $$\frac{1}{n_1}\sum_{i:A_i=1 }Y_i=\frac{1}{n_1}\sum_{i:A_i=1 }\cL(\cA(\dtrain),X_i).
    $$
    Like before, we let $\bar a \in \{0,1\}^n$ with $n_1(\bar a),n_0(\bar a) \neq 0$. We define
    $$
    \mu_{\bar a}^{(1)}(\overline X_{\bar a}) := \frac{1}{n_1(\bar a)}\sum_{i: a_i=1 }\cL(\cA(\cD \cup \overline X_{\bar a}),X_i)
    $$
    Thanks to $\beta$-uniform training stability, it holds that for any adjacent $\overline X_{\bar a}'$: 
    \begin{align*} 
    |\mu_{\bar a}^{(1)}(\overline X_{\bar a})-\mu_{\bar a}^{(1)}(\overline X_{\bar a}')| &\leq \frac{1}{n_1(\bar a)}\sum_{i: a_i=1 } |\cL(\cA_\theta(\cD \cup \overline X_{\bar a}),X_i)-\cL(\cA_\theta(\cD \cup \overline X_{\bar a}'),X'_i)| \\
    &\leq \frac{1}{n_1(\bar a)} + \beta.
    \end{align*}
    Using again McDiarmid's inequality, we find that
    $$
    \proba{\left|\mu_{\bar a}^{(1)}(\overline X_{\bar a})-\E[Y_i|\overline A=\bar a] \right| > \left(\frac{1}{n_1(\bar a)} + \beta\right)\sqrt{\frac{n_1(\bar a) t}{2}} ~\middle| ~\overline A = \bar a} \leq 2e^{-t},
    $$
    where $i$ is such that $a_i=1$. 
    Like in the previous case, we show that, using this time $\beta$-uniform training stability
    $$
    \proba{|\E[Y_i|A_i=1,\overline A_{-i}]-\E[Y_i|A_i=1]| > \sqrt{\frac{(n-1) t \beta^2}{2}}}\leq 2e^{-t}.
    $$
    which allows to conclude.
\end{proof}

\subsection{Proof of \Cref{thm:under_1_run}}

\begin{proof}
    
    Using the same reasoning as in the proof with $0$-uniform training stability, we can handle the term with $A_i=0$. 
    For $\bar a\in \{0,1\}^n$, we let $\theta \in \Theta$ be a generic model parameter, and $\hat \theta_{\bar a}$ be the one obtained after training on $\cD \cup \overline X_{\bar a}$. We define
    $$
    f_\theta(X) := \cL(\theta, X),
    $$
    so that in particular
    $Y_i = f_{\hat \theta_{\bar a}}(X_i)$ on the event $\overline A = \bar a$. Note that $\hat \theta_{\bar a}$ is $\overline X_{\bar a}$-measurable. Because $\theta \mapsto f_{\theta}(x)$ is $L$-Lipschitz for all $x \in \cX$, the empirical process 
    $$
    \theta \mapsto \xi_{\bar a}(\theta) := \frac{1}{n_1(\bar a)} \sum_{a_i=1} f_\theta(X_i)- \E_{\cP_T}[f_\theta(X)]
    $$
    is $2L$-Lipschitz and centered.
    Letting $\cN_\ve$ be a minimal $\ve$-net over ${\rm B}_{\R^D}(0,B)$, if holds that $|\cN_\ve| \leq (2B/\ve)^D$, hence
    $$
    \proba{\sup_{\theta\in\Theta}|\xi_{\bar a}(\theta)| > s ~\middle |~ \overline A = \bar a} \leq 2 \left(\frac{2B}{\ve}\right)^D \exp\left(-\frac{n_1(\bar a)(s-2L\ve)^2}{2}\right).
    $$
    Letting 
    $$
    s = \sqrt{\frac{2\lambda_t}{n_1(\bar a)}}+2L\ve \quad \text{and} \quad \ve = \frac{1}{2L}\sqrt{\frac{2\lambda_t}{n_1(\bar a)}} \quad \text{with} \quad \lambda_t := t + \frac12D \log n_1(\bar a) + D \log(4BL),
    $$
    one finds
    $$
\proba{\left|\frac{1}{n_1(\bar a)} \sum_{a_i=1} Y_i - \E_{\cP_T}[f_{\hat \theta_{\bar a}}(X)|\overline X_{\bar a}, \overline A= \bar a]\right| > 2 \sqrt{\frac{2\lambda_t}{n_1(\bar a)}} ~\middle |~ \overline X_{\bar a}, \overline A = \bar a} \leq 2e^{-t}.
$$
   Now notice that $\E_{\cP_T}[f_{\hat \theta_{\bar a}}(X)|\overline X_{\bar a}, \overline A= \bar a] = \mu_{\bar a}^{(0)}(\overline X_{\bar a})$. Using a previous bound, we show that 
   $$
   \proba{\left|\frac{1}{n_1(\bar a)} \sum_{a_i=1} Y_i - \frac{1}{n_0(\bar a)} \sum_{a_i=0} Y_i \right| > 2 \sqrt{\frac{2\lambda_t}{n_1(\bar a)}} + \sqrt{\frac{t}{2n_{0}(\bar a)}}~\middle |~ \overline X_{\bar a}, \overline A = \bar a} \leq 4e^{-t}.
   $$
   For this we conclude easily, since $|\hat \tau_{\rm ATE}|$ is upper-bounded by $1$:
   $$
   |\tau_{\rm ATE}| = |\E[\hat \tau_{\rm ATE}]| \leq \E[|\hat \tau_{\rm ATE}|] \leq \ve + \P(|\hat \tau_{\rm ATE}|> \ve), \quad \text{for all}~\ve >0.
   $$
\end{proof}

\section{Proofs in the observational setting}

\begin{lemma}
\label{lem:err_pi_tau}
    Let 
    \begin{equation*}
        \bar\tau_{\rm ATE} := \frac{1}{n_1}\sum_{A_i=1} Y_i-\frac{1}{n_0}\sum_{A_i=0}\frac{\pi(X_i)}{1-\pi(X_i)} Y_i\,.
    \end{equation*}
    It holds, with probability $1-2e^{-t}$:
    \begin{equation*}
    |\hat\tau_{\rm ATE}-\bar\tau_{\rm ATE}|\preceq \Delta_{\hat \pi} + \frac{1}{\eta}\sqrt{\frac{t}{n}}\,.
    \end{equation*}

\end{lemma}
\begin{proof}
    We have:
    \begin{align*}
        |\hat\tau_{\rm ATE}-\bar\tau_{\rm ATE}| &= \left|\frac{1}{n_0}\sum_{A_i=0} \left\{
        \frac{\hat\pi(X_i)}{1-\hat\pi(X_i)}-\frac{\pi(X_i)}{1-\pi(X_i)}
        \right\}Y_i\right|\\
        &\leq
        \Delta_{\hat\pi}+ \underbrace{\frac{1}{n_0}\sum_{A_i=0}\left|
        \frac{\hat\pi(X_i)}{1-\hat\pi(X_i)}-\frac{\pi(X_i)}{1-\pi(X_i)}
        \right|-\Delta_{\hat \pi}}_{:= \xi}\,.
    \end{align*}
    Now like in the previous section, we control $\xi$ by conditioning on $\bar A= \bar a$, applying Höeffding's inequality for rvs in $[0,C/\eta]$, and finally use the fact that $n_0 \geq n/4$ w.h.p.
\end{proof}

All the proofs in the observational setting mimic the structures of those of the previous section. Thanks to \Cref{lem:err_pi_tau}, one only needs to bound $|\bar\tau_{\rm ATE}-\tau_{\rm ATE}|$. To ease the notation, we introduce $\zeta(X_i) := \frac{\pi(X_i)}{1-\pi(X_i)}$. As the term involving the $Y_i$'s for $A_i=1$ is the same as in the previous section, we only focus on bounding 
$$
\frac1{n_0} \sum_{A_i=0} \zeta(X_i) Y_i - \E[\zeta(X_i)Y_i|A_i=0].
$$
The bound is the same for each settings and the proof goes as follows. We let again $a \in \{0,1\}^n$ with $n_0(\bar a) \neq 0$, and define, letting $\overline X_{\bar a} := \{X_j\}_{\{j:~a_j=1\}}$,
    $$
    \mu_{\bar a}^{(0)}(\overline X_{\bar a}) := \esp{\zeta(X_i)Y_i|  \bar A= \bar a, \overline X_{\bar a}},
    $$
    for any $i$ such that $a_i=0$. 
    Thanks to Hoeffding's inequality, it holds
    \begin{align*}
        \proba{\left|\frac{1}{n_0(\bar a)}\sum_{i:a_i=0 }\zeta(X_i)Y_i - \mu_{\bar a}^{(0)}(\overline X_{\bar a})\right| > \sqrt{\frac{t}{2n_0(\bar a)\eta^2}} ~\middle|~ \overline A = \overline a, \overline X_{\bar a}}\leq 2e^{-t}\,.
    \end{align*}
    Now for two adjacent datasets $\overline X_{\bar a}$ and $\overline X_{\bar a}'$, it holds:  
    \begin{align*}
        |\mu_{\bar a}^{(0)}(\overline X_{\bar a})-\mu_{\bar a}^{(0)}(\overline X'_{\bar a})|&=|\E_\cA\E_{\cP_0} [\zeta(X)\cL(\cA(\cD \cup \overline X_{\bar a}), X)]-  \E_\cA\E_{\cP_0} [\zeta(X)\cL(\cA(\cD \cup \overline X'_{\bar a}), X)]| \\
        &= |\E_\cA\E_{\cP_T} [\cL(\cA(\cD \cup \overline X_{\bar a}), X)]-  \E_\cA\E_{\cP_T} [\cL(\cA(\cD \cup \overline X'_{\bar a}), X)]|\\
        &\leq \alpha,
    \end{align*}
   thanks to the average stability assumption.
    Thus, using McDiarmid's inequality, it holds 
    \begin{equation*}
        \proba{|\mu_{\bar a}^{(0)}(\overline X_{\bar a})-\E[\zeta(X_i)Y_i|\overline A=\bar a]| > \sqrt{\frac{n_1(\bar a) t \alpha^2}{2}} ~\middle|~ \overline A = \bar a}\leq 2e^{-t}.
    \end{equation*}
    Thus
    \begin{equation*}
        \proba{|\bar \tau_{\rm ATE}-\E[\zeta(X_i)Y_i|\overline A = \bar a]| \leq \sqrt{\frac{2t}{n_0(\bar a)\eta^2}} + \sqrt{\frac{n_1(\bar a)t\alpha^2}{2}} ~\middle | ~ \overline A = \bar a} \geq 1 - 4e^{-t}.
    \end{equation*}
    We conclude the proof by introducing 
    $$
    g(\bar a_{-i}) := \E[\zeta(X_i)Y_i|A_i=0, \overline A_{-i} = a_{-i}].
    $$
    The previous bounds allow for a control of $|\bar \tau_{\rm ATE}- g(\bar A_{-i})|$. Using again the $\alpha$-stability, one find that for two adjacent vectors $\bar A_{-i}, \bar A_{-i}'$, it holds $|g(\bar A_{-i})-g(\bar A_{-i}')| \leq \alpha$. Using McDiarmid's inequality again, we conclude
    $$
    \proba{|g(\overline A_{-i})-\E[\zeta(X_i)Y_i|A_i=0]| > \sqrt{\frac{(n-1) t \alpha^2}{2}}}\leq 2e^{-t}.
    $$
    Piecing all the inequalities together yields the result.

\section{Additional Experimental Analysis}
\label{app:exp_analysis}

\subsection{Additional details on synthetic data experiment}
\label{app:synthetic_exps}

\paragraph{Data generation process.}
Training, member and non-member data is generated as follows.
Training datapoints $X_i=(a_i,b_i)$ are generated with $a_i\sim\cN(0,I_d)$ and $b_i\sim\cN(a_i^\top w_\star,1)|a_i$, for some random fixed unitary vector $w_\star\in\R^d$. This distribution corresponds to $\cP_T$.
Non-member datapoints taken from the same distribution (no distribution shift), or taken with $a_i\sim\cN(\mu,I_d)$ for some $\mu\in\R^d$, depending on the setting (inducing distribution shift between members and non-members). This shifted distribution corresponds to $\cP_0$.
We consider linear regression with the square loss on such data, and consider the loss-based MIA (that thresholds loss values).

\paragraph{Simulation parameters.}
Data simulation parameters are as follows:
\begin{enumerate}
    \item Dimension $d$, which is not the same in the two training algorithms considered next;
    \item Training dataset of size 2000, made of samples from $\cP_T$;
    \item $\mu$ is a random unitary vector;
    \item $w_\star$ is a random standard Gaussian Vector conditioned on having $w_\star^\top \mu/\NRM{\mu} =0.90$ (correlation between $w_\star$ and $\mu$).
\end{enumerate}

\paragraph{Training algorithms $\cA$.}
We consider the two different algorithms to attack for the MIA:
\begin{enumerate}
    \item 
    We first consider a setting where our assumptions hold and for which MIAs should have good-to-moderate performances: overparametrized ridge regression, with $\lambda=10^3$ for regularization parameter, with an ambiant dimension $d=2500$. \Cref{def:error_stability,def:delta_interpolation} will hold in this setting due to uniform stability. 
    \item We then consider a setting for which MIAs should not be able to detect any signal and for which our assumptions hold.
    We thus consider a gradient-based optimizer that adds Gaussian noise at each step and clips gradient of norm larger than 1, therefore limiting sensitivity of each stochastic gradient batch. We add $\ell^2$ regularization with $\lambda=10$, and $d=400$ in this setting.
    The optimizer is DP-SGD, with Gaussian noise $\cN(0,3I_d)$, learning rate of 0.01, 75 epochs and batchsize of 128.
    This corresponds to DP parameters $(\eps,\delta)=(0.602,0.01)$. Stability will hold in this setting, while interpolation will not (due to DP-SGD noise). $\beta$-uniform training stability will hold, as a consequence of uniform stability.
\end{enumerate}

\paragraph{MIA evaluation settings.}
We consider the three different MIA evaluation settings, corresponding to \Cref{alg:n,alg:one,alg:zero}.
\begin{enumerate}
    \item \textit{Multi Run setting}: 400 training of each algorithm to attack. For each run, the training dataset is made of 2000 samples from $\cP_T$ (that stay fixed accross all runs), together with another independent sample from $\cP_T$ for the first 200 runs (members) and none for the last 200 runs. 200 non-members datapoints are then sampled from $\cP_T$.
    \item \textit{$1-$Run setting}, with only one training run. 2000 samples are drawn from $\cP_T$, forming the training set on which the training algorithm is runned. These 2000 training points are the member dataset. 2000 non-members are then independently sampled from the same distribution $\cP_T$;
    \item \textit{$0-$Run setting}: training set and members are the same as in the $1-$Run setting. The 2000 non-members are here sampled from thethe shifted distribution $\cP_0$.
\end{enumerate}
In the multi-run setting there are 200 members and 200 non-members, while in the one-run and zero-run settings there are 2000 of each.

\paragraph{Results.}
Results for all metrics are reported in \Cref{tab:results_synthetic}, ROC curves in \Cref{fig:ROC_curves_merged} (A-B).
\textit{Raw} metrics mean that they were not adjusted by propensity scores while \textit{Corrected} means that they were (if not specified, correction with true propensity scores, \textit{Learned}: correction with propensity scores estimated via logistic regression).
Raw metrics in the 1-Run and 0-Run settings are thus the causal metrics, as well as the corrected ones in the zero-run setting. The raw metrics in the zero-run setting are \textit{not} causal metrics, and our experiments highlight the bias induced by not taking into account distribution shifts.
Confidence intervals are made over 100 runs of independent samples from the same data-generating process.

\paragraph{Interpretation.}
In both settings (DP-SGD, Ridge), \Cref{fig:ROC_curves_merged} (A-B) shows the ROC curves of the Multi-Run and One-Run settings, that don't have any distribution shift (resp. blue and orange).
Then, the ROC curves of the Zero-run setting are shown without correction (green) and with propensity score correction (red).
In all these plots, the non-drifted multi-run and one-run ROC curves are essentially the same: interferences are negligible in this case. Due to distribution shift, the zero-run curve shows easier attack surfaces. This is actually a false interpretation since after correction, the corrected zero-run curve matches that of the non-drifted curves.
As expected, MIAs are more performant in the Ridge setting than the DP-SGD setting.

In \Cref{fig:ROC_curves_merged} (B), we plotted the $y=\min(1,e^\eps x +\delta,1-e^{-\eps}(1-\delta-x))$ curve. Under $(\eps,\delta)-$DP, a MIA cannot have a ROC curve that crosses that black line in the multi-run setting \citep{Dong22GDP}. Here, we observe that not adjusting for distribution shift breaks this and yields a ROC curve above this line.

The metrics in \Cref{tab:results_synthetic} show the same results through four different causal metrics introduced in \Cref{sec:mia_perf_causal}.
Correction to target the right causal metrics make the zero-run performances match the non-drifted settings.
For all metrics and in all settings, the IPW correction makes the causal metrics all align together.

\begin{table*}[h!]
\scriptsize
\centering
\begin{tabular}{l l r r r r r r r r r r r}
\toprule
Model & Evaluation & AUC & $\sup_t \tau_\mathrm{ATE}(t)$ & ATE & TPR0.2 \\
\midrule
Ridge & Multi-run (Raw) & $0.561 \pm 0.016$ & $0.102 \pm 0.023$ & $-287.694 \pm 248.990$ & $0.241 \pm 0.027$ \\
Ridge & One-run (Raw) & $0.561 \pm 0.019$ & $0.103 \pm 0.027$ & $-293.668 \pm 189.226$ & $0.243 \pm 0.031$ \\
Ridge & Zero-run (Raw) & $0.669 \pm 0.020$ & $0.269 \pm 0.031$ & $-1049.710 \pm 138.808$ & $0.350 \pm 0.041$ \\
Ridge & Zero-run (Corrected, Oracle) & $0.562 \pm 0.026$ & $0.108 \pm 0.036$ & $-297.327 \pm 151.472$ & $0.201 \pm 0.068$ \\
Ridge & Zero-run (Corrected, Learned) & $0.595 \pm 0.025$ & $0.159 \pm 0.034$ & $-496.127 \pm 106.100$ & $0.228 \pm 0.083$ \\
DP-SGD & Multi-run (Raw) & $0.504 \pm 0.018$ & $0.024 \pm 0.024$ & $0.633 \pm 50.951$ & $0.207 \pm 0.024$ \\
DP-SGD & One-run (Raw) & $0.503 \pm 0.017$ & $0.023 \pm 0.020$ & $-0.949 \pm 46.320$ & $0.205 \pm 0.022$ \\
DP-SGD & Zero-run (Raw) & $0.616 \pm 0.016$ & $0.183 \pm 0.026$ & $-322.243 \pm 41.150$ & $0.299 \pm 0.032$ \\
DP-SGD & Zero-run (Corrected, Oracle) & $0.503 \pm 0.024$ & $0.029 \pm 0.026$ & $-1.985 \pm 68.971$ & $0.165 \pm 0.083$ \\
DP-SGD & Zero-run (Corrected, Learned) & $0.512 \pm 0.059$ & $0.063 \pm 0.084$ & $-23.022 \pm 158.351$ & $0.174 \pm 0.111$ \\
\bottomrule
\end{tabular}
\caption{Causal metrics in the three settings (with oracle or learned propensity scores in the zero-run setting), and non-causal metrics in the zero-run setting to highlight the distribution-shift induced bias.
$\tau_\mathrm{ATE}(t)$ is defined as $\tau_\mathrm{TPR}(t)-\tau_\mathrm{FPR}(t)$ (Youden's J metric at a given threshold), and $\sup_t\tau_\mathrm{ATE}(t)$ is thus the largest gap between (causal) TPR and FPR.
TPR0.2 is the TPR at FPR $\alpha=0.2$.
}
\label{tab:results_synthetic}
\end{table*}

\subsection{Additional details on CIFAR10}

We now present in more details our CIFAR10 experimental setup.

\textbf{Dataset Construction.} We define two base domains: the original clean CIFAR-10 data ($\cD_{\text{clean}}$) and a corrupted domain ($\cD_{\text{noise}}$) generated by adding Gaussian noise with variance $\sigma=0.2$ to the images. To ensure valid overlap, we construct the datasets as follows:
 i) The Members ($\cD_T$) training set is constructed such that $90\%$ of samples are drawn from $\cD_{\text{clean}}$ and $10\%$ from $\cD_{\text{noise}}$ and ii) The non-members ($\cD_0$) are drawn from a distribution of $10\%$ samples from $\cD_{\text{clean}}$ and $90\%$ from $\cD_{\text{noise}}$.
This design creates a covariate shift: members are predominantly clean, while non-members are predominantly noisy, yet it guarantees that the density ratio remains bounded, since every sample has a non-zero probability of appearing in both sets.

\textbf{Target Model.} On this dataset, we train a ResNet-16 \cite{he2016deep} architecture trained on. The model is optimized using Stochastic Gradient Descent (SGD) with a batch size of 128 and a fixed learning rate of 0.01. Training is performed for 10000 optimization steps. To isolate the effects of distribution shift and memorization without additional confounders, we train the model without any data augmentation.

\textbf{Propensity Score Estimation.}
To correct for this distribution shift, we estimate the propensity score $\pi(x) = \P(A=1|X=x)$. We treat this as a binary classification problem, distinguishing between the member distribution $\cD_T$ and the non-member distribution $\cD_0$. We fine-tune a ResNet model \cite{he2016deep} pre-trained on ImageNet \cite{5206848} to minimize the binary cross-entropy loss. We perform this estimation using cross-validation: the propensity model is fitted on a held-out portion of the data mixtures, ensuring that the estimated weights $\hat{\pi}(X_i)$ are independent of the MIA scores $Y_i$ used for the final evaluation.

\textbf{Results.} We compare the performance of Membership Inference Attacks across the three evaluation regimes. \Cref{fig:ROC_curves_merged} (C) displays the ROC curves for the loss-based attack.
The \textit{IID (Test)} curve (blue) represents the idealized setting where $\cD_0$ follows the same distribution as the training set ($\cD_T$). The \textit{Drifted} curve (orange) shows the naive zero-run evaluation where non-members are drawn from the shifted (noisy) distribution. This induces an inflation in attack performance ($\text{AUC} \approx 0.75$), misleadingly suggesting high memorization.
However, this advantage is spurious: it stems from the model's high loss on noisy non-members rather than true privacy leakage. By applying our propensity score correction (green curve), we recover the true causal memorization signal ($\text{AUC} \approx 0.66$), which aligns closely with the valid IID baseline ($\text{AUC} \approx 0.65$). This confirms that our method successfully removes the bias.

\begin{figure}[h]
    \centering
    \begin{subfigure}{0.48\textwidth}
        \centering
        \includegraphics[width=\linewidth]{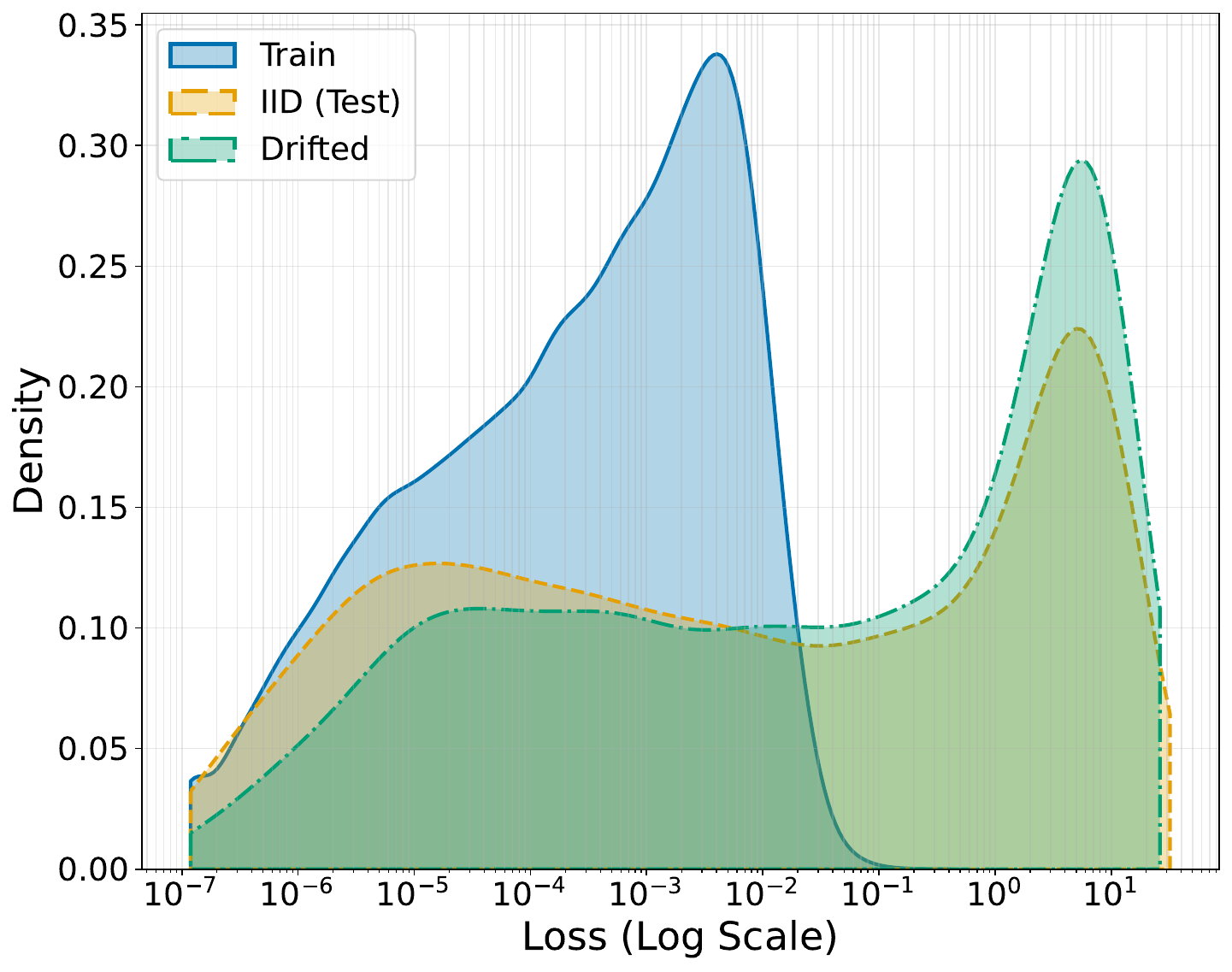}
        \caption{Loss Distributions}
        \label{fig:loss_dist}
    \end{subfigure}
    \hfill
    \begin{subfigure}{0.48\textwidth}
        \centering
        \includegraphics[width=\linewidth]{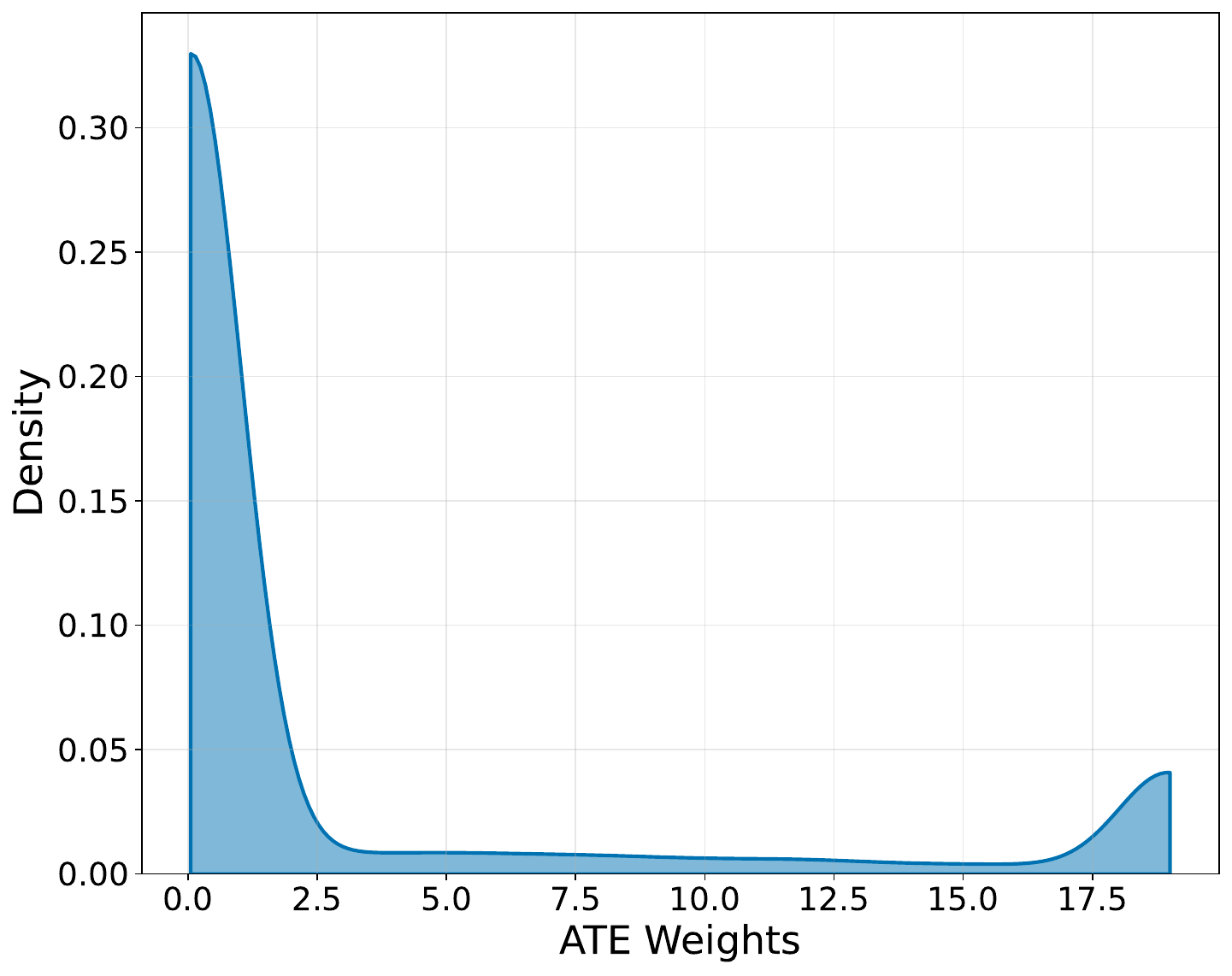}
        \caption{Propensity Weights}
        \label{fig:weights_dist}
    \end{subfigure}
    \caption{\textbf{Additional metrics.} \textbf{(a)} The loss distribution for drifted non-members is shifted, making them more  distinguishable from members without true memorization. \textbf{(b)} The distribution of estimated ATE weights shows how our method corrects this by upweighting the minority of non-members that are distributionally similar to the training set.}
    \label{fig:bias_analysis}
\end{figure}

To better understand the source of bias in the zero-run setting, we analyze the distributions of losses and the estimated propensity weights for the CIFAR10 case. Figure~\ref{fig:loss_dist} visualizes the density of loss values for training members (Train), IID test non-members, and Drifted non-members. The \textit{Drifted} distribution (orange) is shifted significantly to the right (higher loss) compared to the \textit{IID} distribution (blue).
Because standard MIAs predict membership based on low loss, the separation between the low-loss members and high-loss drifted non-members is artificially large, leading to the inflated AUC observed in the main text. The overlap between the Train and IID Test distributions is bigger, reflecting the attack's true performance.

\textbf{Propensity Weights.}
Figure~\ref{fig:weights_dist} shows the distribution of the inverse probability weights (IPW) applied to the non-member samples during our correction. The distribution is heavy-tailed, assigning high weights to the rare non-members that resemble the training distribution (low noise) and low weights to the abundant noisy samples.

\subsection{Additional Details on the Pythia LLM Experiment}
\label{app:pythia}

\begin{figure}
    \centering
    \includegraphics[width=\linewidth]{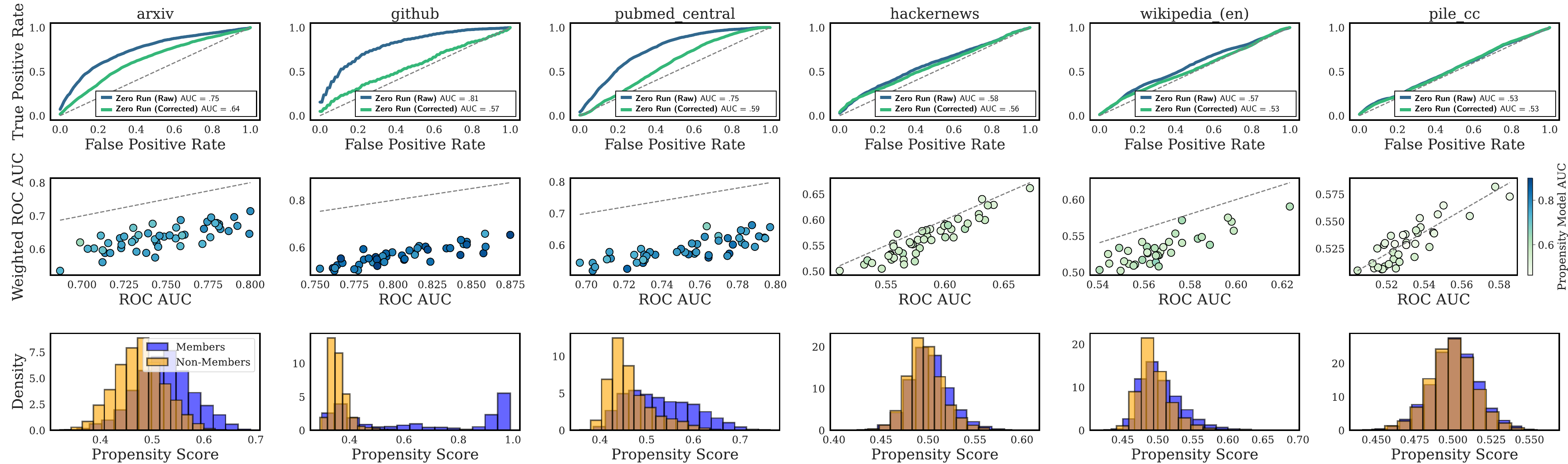}
    \caption{Results of the MIA experiment on the \texttt{Pythia-12b} model. The \textbf{first row} displays the ROC curve for 6 different datasets corresponding to the \textbf{columns}. The \textbf{second row} compares the unweighted (x-axis) to the weighted ROC AUC (y-axis). Every dot represents a different test-train split of the MIA data, the color indicating the AUC of the propensity score model. Lastly, the \textbf{third row} shows the distribution of the propensity score, which is a proxy for distribution shift between members and non-members.}
    \label{fig:enter-label}
\end{figure}

\paragraph{Datasets and Model.} Building on \citet{meeus2024sok}, we consider the Pythia suite of models, specifically the \texttt{12b} parameters model. This model been trained on subsets of the MIMIR dataset with information about members and non-members being openly available. We consider specifically $6$ subsets for which membership information is available (see Table \ref{table:details_mimir} for details on these datasets).

\begin{table}[]
    \centering
\begin{tabular}{ccc}
\toprule
Dataset & Members & Non-members \\
\midrule
arxiv & 1000 & 500 \\
github & 1000 & 268 \\
hackernews & 1000 & 646 \\
pubmed\_central & 1000 & 491 \\
pile\_cc & 1000 & 1000 \\
wikipedia\_(en) & 1000 & 1000 \\
\bottomrule
\end{tabular}
\vspace{5pt}
\caption{Number of members and non-members for every corpus of MIMIR. }
\label{table:details_mimir}
\end{table}

\begin{figure}
    \centering
    \includegraphics[width=\linewidth]{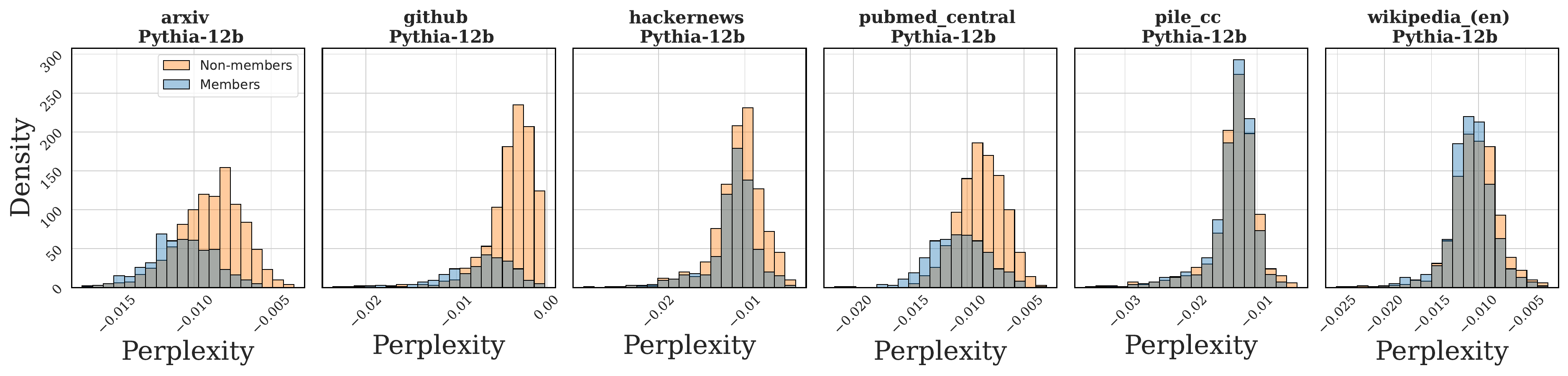}
    \caption{Perplexity for members and non-members for \texttt{Pythia-12b}. }
    \label{fig:enter-label2}
\end{figure}

\begin{figure}
  \begin{center}
    \includegraphics[width=0.5\textwidth]{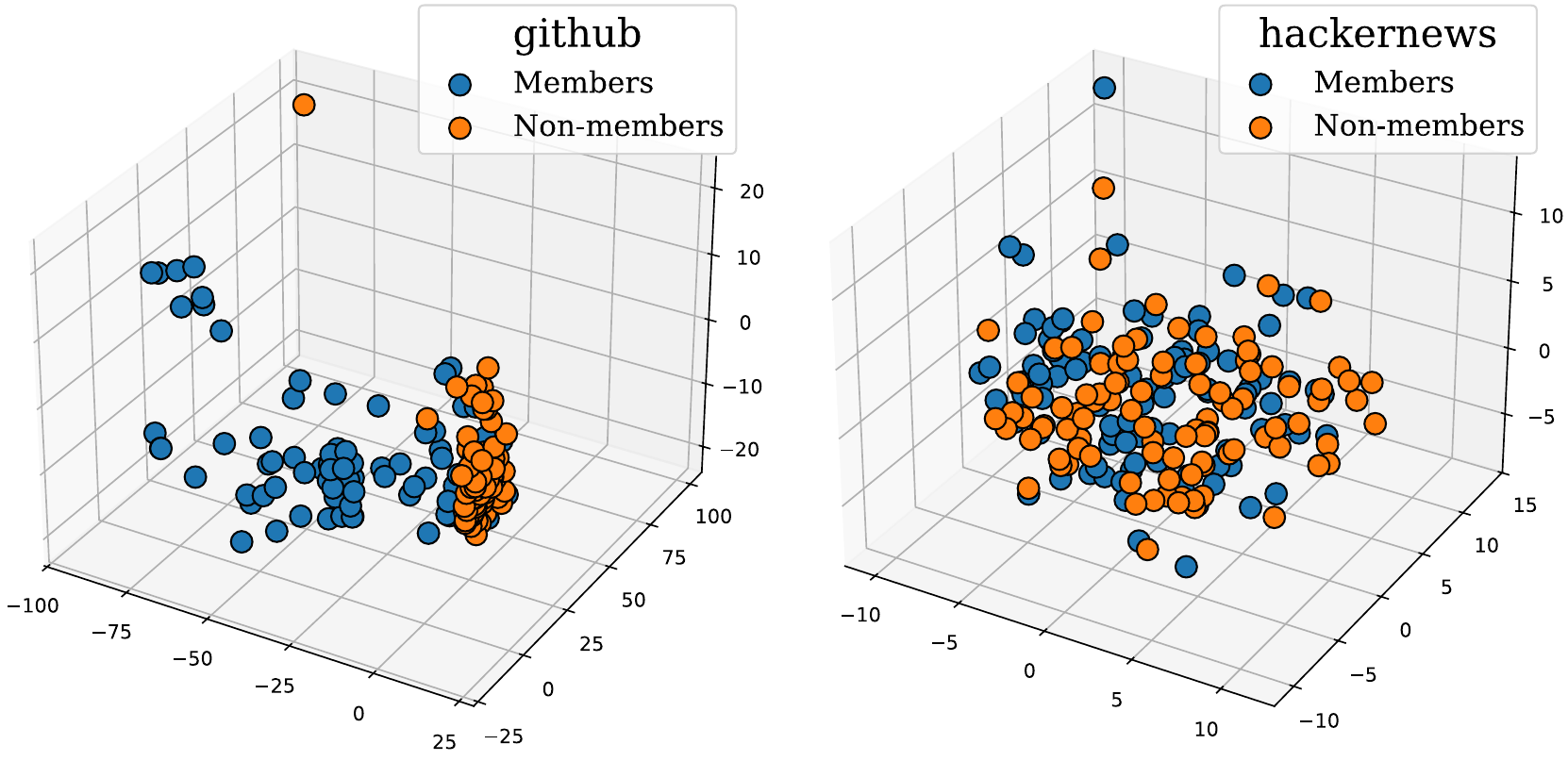}
  \end{center}
  \caption{Mean-pooled embbedings produced by \texttt{Pythia-70m} of elements of \texttt{github} v.s. \texttt{hackernews}.}
  \label{fig:3d_pca}
\end{figure}

\paragraph{Inference of the Propensity Score.} To estimate the propensity score $\pi(\cdot)$ for every text, we follow  \citet{meeus2024sok} and use a bag-of-words tokenizer in combination with a random forest classifier. The classifier is trained on $80\%$ of the data, and we use the remaining $20\%$ for evaluation. Note that we balance all datasets prior to training to have the same number of members and non-members. Propensity scores are clipped to the range $[0.01, 0.99]$ to avoid numerical errors. Exact details of the tokenizer and the propensity score model are given in Table \ref{tab:details_propensity_score}.

\begin{table}[]
    \centering
    \begin{tabular}{ccc}
    \toprule
        \textbf{Parameter} & \textbf{MIMIR} & \textbf{Trump-Biden} \\
        \toprule
        \texttt{Tokenizer} & \texttt{CountVectorizer} & --- \\
        \texttt{min\_df} & $0.01$ & $0.13$\\
        \texttt{max\_df} & $1$ & ---\\
        \texttt{max\_features} & $1000$ & ---\\
        \midrule 
        \texttt{Propensity Score Model} & \texttt{RandomForest}& ---\\
        \texttt{n\_estimators} & $100$ & ---\\
        \texttt{max\_depth} & $5$ & ---\\
        \texttt{min\_samples\_leaf} & $5$ & ---\\
        \bottomrule
    \end{tabular}
    \vspace{5pt}
    \caption{Parameters for propensity score inference on the MIMIR and Trump-Biden datasets. We report two values for parameters only if they differ. }
    \label{tab:details_propensity_score}
\end{table}

\paragraph{Membership attacks.} On the testing split, we query the perplexity of the model $p_i$ for every datapoint $X_i$.
We then compute the ROC curve for both the standard and reweighted perplexity. 

\subsection{Additional Details on White-House Press Release LLM Experiment}
\label{app:trump_biden}

\paragraph{Data scrapping.} We scrap the titles of all press statements made by the Biden and Trump II administration. For the Biden administration, we scrap all titles starting on the the $15$th of January 2021 from the website \texttt{https://bidenwhitehouse.archives.gov}. For the Trump II administration, we scrap all titles starting from the $20$th of January 2025 from \texttt{https://www.whitehouse.gov/news/}. We will make this data publicaly available. A 2d-PCA of the data is displayed in Figure \ref{fig:trump_biden_pca} and shows a clear distribution shift between the two corpuses. We additionaly display a Umap in Figure \ref{fig:umap_biden_trump} along with the most frequent words in every cluster. 

\begin{figure}
    \centering
    \includegraphics[width=0.6\linewidth]{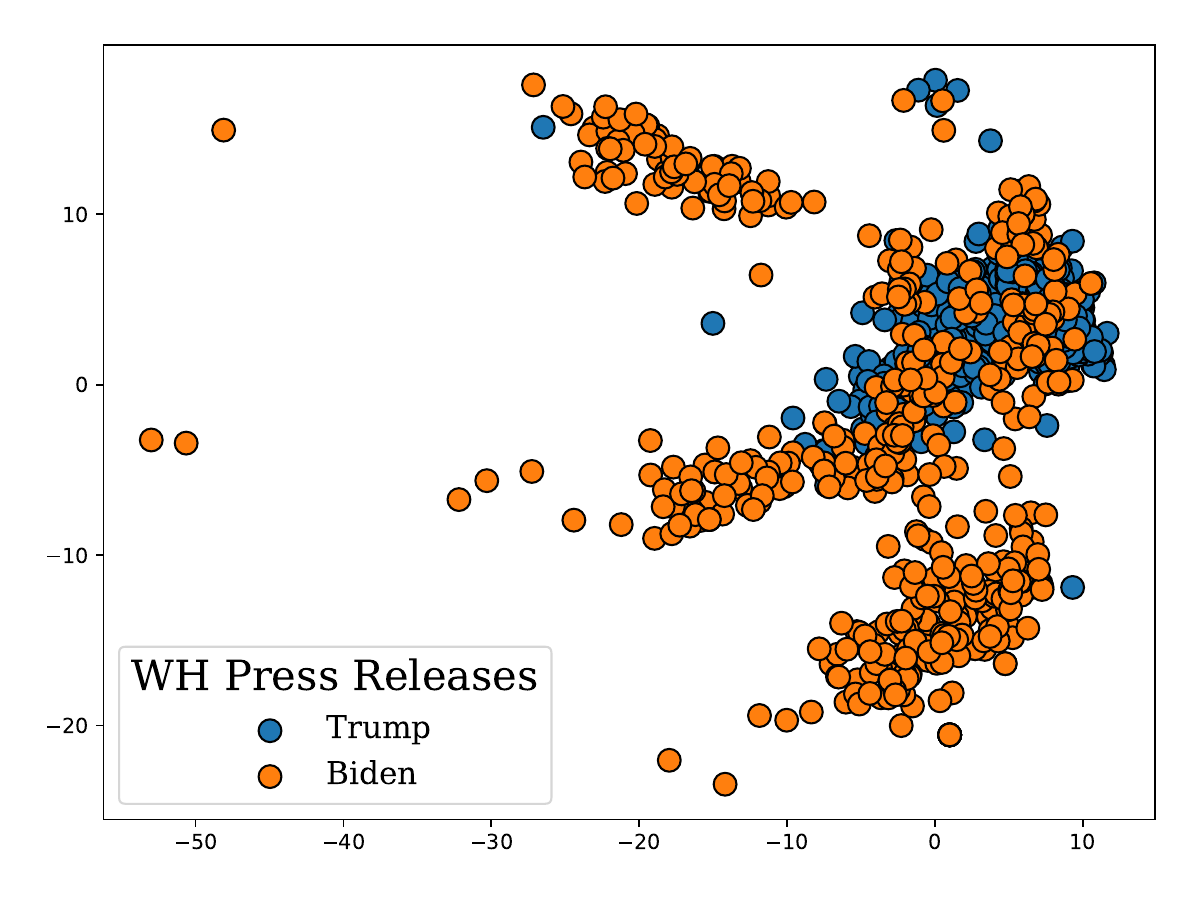}
    \caption{PCA of the mean-pooled internal representation of the non-fine-tuned \texttt{Pythia-70m} model of the two corpuses.}
    \label{fig:trump_biden_pca}
\end{figure}

\begin{figure}
    \centering
    \includegraphics[width=0.45\linewidth]{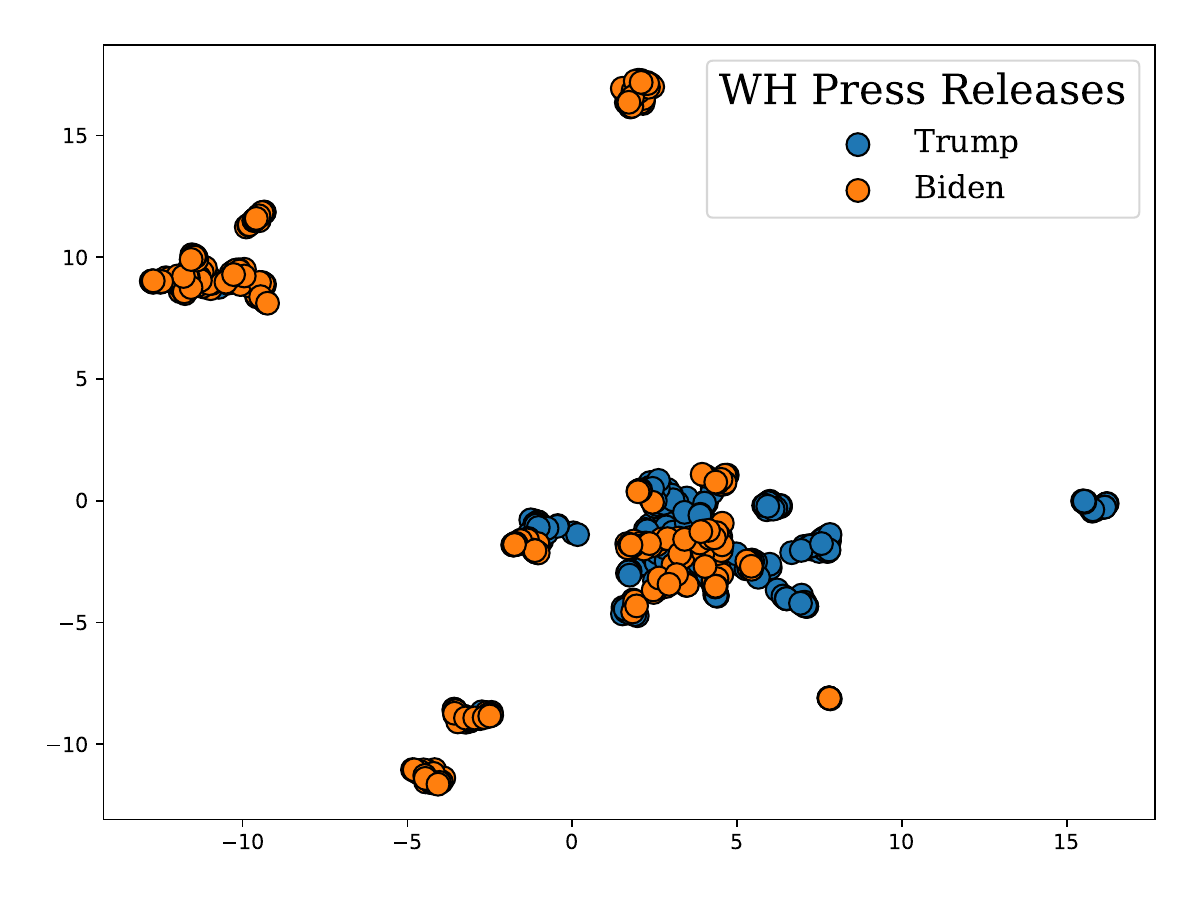}
    \includegraphics[width=0.45\linewidth]{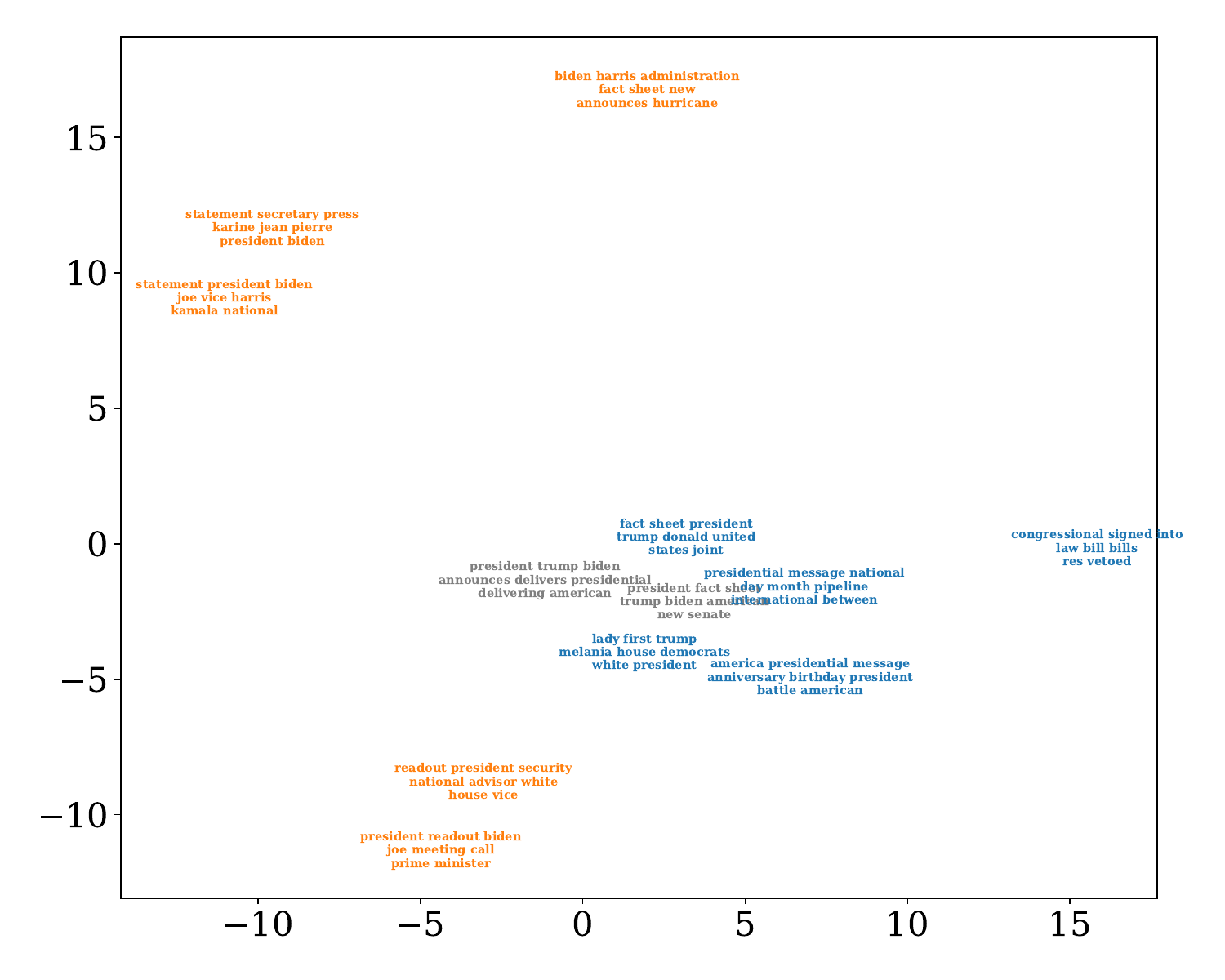}
    \caption{Umap of the titles of press releases from the Biden administration (in orange) and the Trump administration (in blue) on the \textbf{left}. The \textbf{right} plot shows the most frequent words for every cluster in the original umap. We color the words in grey if the percentage of titles in this cluster to belong to the Trump administration is between $30\%$ and $70\%$. If it is bellow $30\%$, the words are colored in orange, and above $70\%$ they are colored in blue.}
    \label{fig:umap_biden_trump}
\end{figure}

\paragraph{Fine-Tuning.} We randomly split the titles from the Biden administration into two halfes of equal size. We then fine-tune the \texttt{Pythia-70m} model for $2$ epoches. Parameters for the fine-tuning are given in Table \ref{tab:fine_tuning}.

\paragraph{Inference of the Propensity Score.} We use the same propensity score model than for the first LLM experiment. Parameters are reported in Table \ref{tab:details_propensity_score}.

\paragraph{Membership attacks.} We conduct the MIAs in the same way than in the previous experiment.

\paragraph{Results.}
Complete results can be found in \Cref{fig:biden_trump_app_total}.

\begin{figure}
    \centering
    \includegraphics[width=\linewidth]{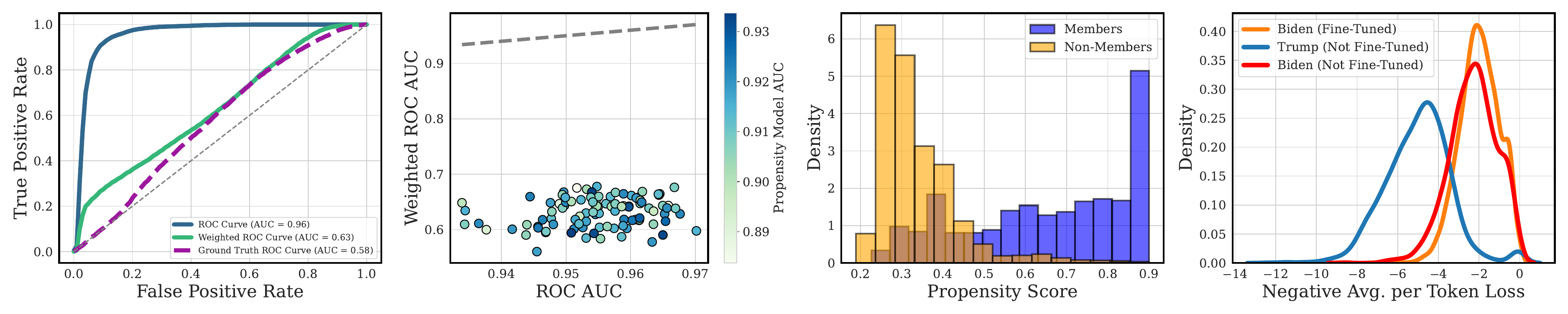}
    \caption{Results of the MIA experiment on the fine-tuned \texttt{Pythia-70m} model. The \textbf{first plot} shows the ROC curves of the vanilla, debiased and pseudo-ground-truth MIAs. The \textbf{second plot} compares the ROC AUC for the vanilla (x-axis) v.s. the debiased MIA (y-axis). Each dot represents one random split of the data into a training and a test set, the color representing the AUC of the propensity score classifier (who classifies based on the text). The \textbf{third plot} shows the propensity scores as infered by the propensity score model. Finally, the \textbf{last plot} shows a KDE plot of the negative per token mean loss of the fine-tuned model for all three datasets.}
    \label{fig:biden_trump_app_total}
\end{figure}

\begin{table}
    \centering
    \begin{tabular}{cc}
    \toprule 
        \textbf{Parameter} & \textbf{Value} \\
        \toprule
         Number of epochs & $2$ \\
         Batch size & $16$ \\
         Weight decay &  $0$\\
         Optimizer & AdamW\\
         Intial Learning Rate & $5\times 10^{-4}$ \\
         Learning Rate Schedule & Linear\\
         \midrule
         \texttt{LoRA} $r$ & $16$ \\
         \texttt{LoRA} $\alpha$ & $32$ \\
         \texttt{LoRA} dropout & $0.05$\\
         Targeted modules & \texttt{q\_proj, v\_proj, query\_key\_value, k\_proj} \\
         \bottomrule
    \end{tabular}
    \vspace{5pt}
    \caption{Parameters for the fine-tuning experiment. Unreported parameters are defaulted to the parameters of \texttt{transformers.TrainingArguments} and \texttt{peft.LoRAConfig}.}
    \label{tab:fine_tuning}
\end{table}

\begin{figure}
    \centering
    \includegraphics[width=\linewidth]{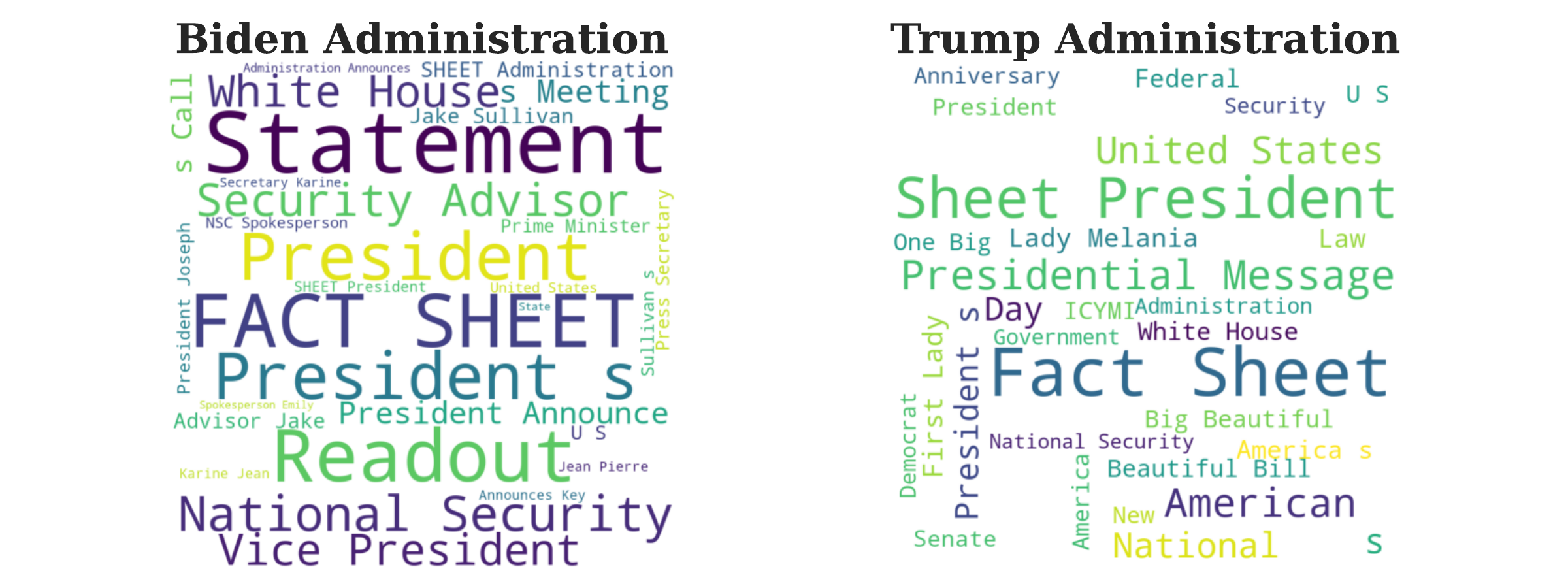}
    \caption{Word clouds for the titles of the press releases of the Biden administration (\textbf{left}) and the Trump I administration (\textbf{right}). To generate these pointclouds, the names of the Presidents and their Vice-Presidents have been masked by [NAME] and [FIRSTNAME] placeholders, which we have redacted from these word clouds for clarity.}
    \label{fig:enter-label3}
\end{figure}


\end{document}

%% file: macros.tex
\usepackage{tikz}
\usetikzlibrary{arrows.meta, positioning}

\usepackage{xcolor}
\definecolor{darkgreen}{RGB}{0,100,0}   
\definecolor{darkred}{RGB}{139,0,0}     

\newcommand{\dtrain}{\mathcal D_\mathrm{train}}
\newcommand{\dtraini}{\mathcal D_\mathrm{train,i}}

\newcommand{\E}{\mathbb{E}}

\newcommand{\indep}{\perp \!\!\! \perp}

\newcommand{\cL}{\mathcal{L}}

\newcommand{\cO}{\mathcal O}

\newcommand{\R}{\mathbb R}

\newcommand{\eqdef}{\stackrel{\text{def}}{=}}

\def\<#1,#2>{\langle #1,#2\rangle}

\usepackage{dsfont}

\def\<{\langle}
\def\>{\rangle}
\def\|{\Vert}

\def\eps{\varepsilon}

\newcommand{\esp}[1]{\mathbb{E}\left[#1\right]}
\newcommand{\espT}[1]{\mathbb{E}_{\cP_T}\left[#1\right]}

\newcommand{\NRM}[1]{{{\left\| #1\right\|}}} 
\newcommand{\set}[1]{{\left\{ #1\right\}}} 
\newcommand{\proba}[1]{\mathbb{P}\left(#1\right)}
\newcommand{\probaT}[1]{\mathbb{P}_{\cP_T}\left(#1\right)}

\renewcommand{\P}{\mathbb{P}}
\newcommand{\cA}{\mathcal{A}}

\newcommand{\cN}{\mathcal{N}}
\newcommand{\cX}{\mathcal{X}}

\newcommand{\dd}{{\rm d}}

\newcommand{\cP}{\mathcal{P}}

\newcommand{\cD}{\mathcal{D}}
\newcommand{\one}{\mathds{1}}

\usepackage{amsfonts}
\usepackage{amsthm}
\usepackage{amsmath}
\usepackage{amssymb}
\usepackage{mathrsfs}
\usepackage{amscd}
\usepackage{caption}
\usepackage{cleveref}

\DeclareMathOperator*{\argmax}{argmax}

%% file: references.bib
@inproceedings{10.1145/3548606.3560663,
author = {Jayaraman, Bargav and Evans, David},
title = {Are Attribute Inference Attacks Just Imputation?},
year = {2022},
booktitle = {PCCS},
}

@inproceedings{10.1145/2810103.2813677,
author = {Fredrikson, Matt and Jha, Somesh and Ristenpart, Thomas},
title = {Model Inversion Attacks that Exploit Confidence Information and Basic Countermeasures},
year = {2015},
booktitle = {CCS},
}

@inproceedings{Keinan2025,
title={How Well Can Differential Privacy Be Audited in One Run?},
author={Amit Keinan and Moshe Shenfeld and Katrina Ligett},
booktitle={The Thirty-ninth Annual Conference on Neural Information Processing Systems},
year={2025},
url={https://openreview.net/forum?id=004uTlSufe}
}

@inproceedings{Xiang2025,
author = {Xiang, Zihang and Wang, Tianhao and Wang, Di},
title = {Privacy audit as bits transmission: (im)possibilities for audit by one run},
year = {2025},
booktitle = {Proceedings of the 34th USENIX Conference on Security Symposium},
}

@article{Liu2025one_run_quantile_regression,
  title = {Enhancing {{One-run Privacy Auditing}} with {{Quantile Regression-Based Membership Inference}}},
  author = {Liu, Terrance and Boglioni, Matteo and Fu, Yiwei and Hu, Shengyuan and Thaker, Pratiksha and Wu, Zhiwei Steven},
  year = {2025},
  journal = {arXiv preprint arXiv:2506:15349}
}

@INPROCEEDINGS{privacy_games,
author = { Salem, Ahmed and Cherubin, Giovanni and Evans, David and Kopf, Boris and Paverd, Andrew and Suri, Anshuman and Tople, Shruti and Zanella-Beguelin, Santiago },
booktitle = {IEEE Symposium on Security and Privacy (S\&P)},
title = {{ SoK: Let the Privacy Games Begin! A Unified Treatment of Data Inference Privacy in Machine Learning }},
year = {2023},}

@article{meeus2024sok,
  title={SoK: Membership Inference Attacks on LLMs are Rushing Nowhere (and How to Fix It)},
  author={Meeus, Matthieu and Shilov, Igor and Jain, Shubham and Faysse, Manuel and Rei, Marek and de Montjoye, Yves-Alexandre},
  journal={arXiv preprint arXiv:2406.17975},
  year={2024}
}

@inproceedings{duan2024membership,
  title={Do Membership Inference Attacks Work on Large Language Models?}, 
  author={Michael Duan and Anshuman Suri and Niloofar Mireshghallah and Sewon Min and Weijia Shi and Luke Zettlemoyer and Yulia Tsvetkov and Yejin Choi and David Evans and Hannaneh Hajishirzi},
  year={2024},
  booktitle={Conference on Language Modeling (COLM)},
}

@book{ImbensRubin2015,
  author    = {Imbens, Guido W. and Rubin, Donald B.},
  title     = {Causal Inference in Statistics, Social, and Biomedical Sciences},
  publisher = {Cambridge University Press},
  address   = {Cambridge, UK},
  year      = {2015},
  isbn      = {978-0521885881}
}

@article{mahloujifar2024auditing,
  title={Auditing $ f $-Differential Privacy in One Run},
  author={Mahloujifar, Saeed and Melis, Luca and Chaudhuri, Kamalika},
  journal={arXiv preprint arXiv:2410.22235},
  year={2024}
}

@inproceedings{yeom2018privacy,
  title={Privacy risk in machine learning: Analyzing the connection to overfitting},
  author={Yeom, Samuel and Giacomelli, Irene and Fredrikson, Matt and Jha, Somesh},
  booktitle={2018 IEEE 31st computer security foundations symposium (CSF)},
  pages={268--282},
  year={2018},
  organization={IEEE}
}

@article{sobel2006randomized,
  title={What do randomized studies of housing mobility demonstrate? Causal inference in the face of interference},
  author={Sobel, Michael E},
  journal={Journal of the American Statistical Association},
  volume={101},
  number={476},
  pages={1398--1407},
  year={2006},
  publisher={Taylor \& Francis}
}

@misc{even2025winratio,
      title={Rethinking the Win Ratio: A Causal Framework for Hierarchical Outcome Analysis}, 
      author={Mathieu Even and Julie Josse},
      year={2025},
      eprint={2501.16933},
      archivePrefix={arXiv},
      primaryClass={stat.ME},
      url={https://arxiv.org/abs/2501.16933}, 
}

@article{Mao2017,
	title        = {On causal estimation using $U$-statistics},
	author       = {Mao,  Lu},
	year         = {2017},
	month        = {dec},
	journal      = {Biometrika},
	publisher    = {Oxford University Press (OUP)},
	volume       = {105},
	number       = {1},
	pages        = {215–220},
	issn         = {1464-3510}
}

@article{Buyse2010,
	title        = {Generalized pairwise comparisons of prioritized outcomes in the two‐sample problem},
	author       = {Buyse,  Marc},
	year         = {2010},
	month        = {dec},
	journal      = {Statistics in Medicine},
	publisher    = {Wiley},
	volume       = {29},
	number       = {30},
	pages        = {3245–3257}
}

@article{Pocock2011winratio,
	title        = {The win ratio: a new approach to the analysis of composite endpoints in clinical trials based on clinical priorities},
	author       = {Pocock,  S. J. and Ariti,  C. A. and Collier,  T. J. and Wang,  D.},
	year         = {2011},
	month        = {sep},
	journal      = {European Heart Journal},
	publisher    = {Oxford University Press (OUP)},
	volume       = {33},
	number       = {2},
	pages        = {176–182},
	issn         = {1522-9645}
}

@article{Wilcoxon1945,
	title        = {Individual Comparisons by Ranking Methods},
	author       = {Wilcoxon,  Frank},
	year         = {1945},
	month        = {dec},
	journal      = {Biometrics Bulletin},
	publisher    = {JSTOR},
	volume       = {1},
	number       = {6},
	pages        = {80},
	issn         = {0099-4987}
}

@article{Mann1947,
	title        = {On a Test of Whether one of Two Random Variables is Stochastically Larger than the Other},
	author       = {Mann,  H. B. and Whitney,  D. R.},
	year         = {1947},
	month        = {mar},
	journal      = {The Annals of Mathematical Statistics},
	publisher    = {Institute of Mathematical Statistics},
	volume       = {18},
	number       = {1},
	pages        = {50–60},
	issn         = {0003-4851}
}

@article{tian2000probabilities,
  title={Probabilities of causation: Bounds and identification},
  author={Tian, Jin and Pearl, Judea},
  journal={Annals of Mathematics and Artificial Intelligence},
  volume={28},
  number={1},
  pages={287--313},
  year={2000},
  publisher={Springer}
}

@article{ogburn2024causal,
  title={Causal inference for social network data},
  author={Ogburn, Elizabeth L and Sofrygin, Oleg and Diaz, Ivan and Van der Laan, Mark J},
  journal={Journal of the American Statistical Association},
  volume={119},
  number={545},
  pages={597--611},
  year={2024},
  publisher={Taylor \& Francis}
}

@inproceedings{carlini2019secret,
  title={The secret sharer: Evaluating and testing unintended memorization in neural networks},
  author={Carlini, Nicholas and Liu, Chang and Erlingsson, {\'U}lfar and Kos, Jernej and Song, Dawn},
  booktitle={28th USENIX security symposium (USENIX security 19)},
  pages={267--284},
  year={2019}
}

@article{bertran2023scalable,
  title={Scalable membership inference attacks via quantile regression},
  author={Bertran, Martin and Tang, Shuai and Roth, Aaron and Kearns, Michael and Morgenstern, Jamie H and Wu, Steven Z},
  journal={Advances in Neural Information Processing Systems},
  volume={36},
  pages={314--330},
  year={2023}
}

@misc{hernan2010causal,
  title={Causal inference},
  author={Hern{\'a}n, Miguel A and Robins, James M},
  year={2010},
  publisher={CRC Boca Raton, FL}
}

@article{rubin1974estimating,
  title={Estimating causal effects of treatments in randomized and nonrandomized studies.},
  author={Rubin, Donald B},
  journal={Journal of educational Psychology},
  volume={66},
  number={5},
  pages={688},
  year={1974},
  publisher={American Psychological Association}
}

@inproceedings{shokri_enhanced,
author = {Ye, Jiayuan and Maddi, Aadyaa and Murakonda, Sasi Kumar and Bindschaedler, Vincent and Shokri, Reza},
title = {Enhanced Membership Inference Attacks against Machine Learning Models},
year = {2022},
booktitle = {CCS},
}

@inproceedings{shokri2017membership,
  title={Membership inference attacks against machine learning models},
  author={Shokri, Reza and Stronati, Marco and Song, Congzheng and Shmatikov, Vitaly},
  booktitle={2017 IEEE symposium on security and privacy (SP)},
  pages={3--18},
  year={2017},
  organization={IEEE}
}

@inproceedings{DBLP:conf/icml/ZarifzadehLS24,
  author       = {Sajjad Zarifzadeh and
                  Philippe Liu and
                  Reza Shokri},
  title        = {Low-Cost High-Power Membership Inference Attacks},
  booktitle    = {Forty-first International Conference on Machine Learning, {ICML} 2024,
                  Vienna, Austria, July 21-27, 2024},
  publisher    = {OpenReview.net},
  year         = {2024},
  url          = {https://openreview.net/forum?id=sT7UJh5CTc},
  bibsource    = {dblp computer science bibliography, https://dblp.org}
}

@inproceedings{carlini2021extracting,
  title={Extracting training data from large language models},
  author={Carlini, Nicholas and Tramer, Florian and Wallace, Eric and Jagielski, Matthew and Herbert-Voss, Ariel and Lee, Katherine and Roberts, Adam and Brown, Tom and Song, Dawn and Erlingsson, Ulfar and others},
  booktitle={30th USENIX security symposium (USENIX Security 21)},
  pages={2633--2650},
  year={2021}
}

@article{shi2023detecting,
  title={Detecting pretraining data from large language models},
  author={Shi, Weijia and Ajith, Anirudh and Xia, Mengzhou and Huang, Yangsibo and Liu, Daogao and Blevins, Terra and Chen, Danqi and Zettlemoyer, Luke},
  journal={arXiv preprint arXiv:2310.16789},
  year={2023}
}

@inproceedings{meeus2024did,
  title={Did the neurons read your book? document-level membership inference for large language models},
  author={Meeus, Matthieu and Jain, Shubham and Rei, Marek and de Montjoye, Yves-Alexandre},
  booktitle={33rd USENIX Security Symposium (USENIX Security 24)},
  pages={2369--2385},
  year={2024}
}

@article{papadogeorgou2019causal,
  title={Causal inference with interfering units for cluster and population level treatment allocation programs},
  author={Papadogeorgou, Georgia and Mealli, Fabrizia and Zigler, Corwin M},
  journal={Biometrics},
  volume={75},
  number={3},
  pages={778--787},
  year={2019},
  publisher={Wiley Online Library}
}

@article{manski2013identification,
  title={Identification of treatment response with social interactions},
  author={Manski, Charles F},
  journal={The Econometrics Journal},
  volume={16},
  number={1},
  pages={S1--S23},
  year={2013},
  publisher={Oxford University Press Oxford, UK}
}

@inproceedings{bhattacharya2020causal,
  title={Causal inference under interference and network uncertainty},
  author={Bhattacharya, Rohit and Malinsky, Daniel and Shpitser, Ilya},
  booktitle={Uncertainty in Artificial Intelligence},
  pages={1028--1038},
  year={2020},
  organization={PMLR}
}

@article{cinelli2025challenges,
  title={Challenges in statistics: A dozen challenges in causality and causal inference},
  author={Cinelli, Carlos and Feller, Avi and Imbens, Guido and Kennedy, Edward and Magliacane, Sara and Zubizarreta, Jose},
  journal={arXiv preprint arXiv:2508.17099},
  year={2025}
}

@article{Dong22GDP,
  title={Gaussian differential privacy},
  author={Dong, Jinshuo and Roth, Aaron and Su, Weijie J},
  journal={Journal of the Royal Statistical Society: Series B (Statistical Methodology)},
  volume={84},
  number={1},
  pages={3--37},
  year={2022},
  publisher={Wiley Online Library}
}

@article{pitcan2017note,
  title={A note on concentration inequalities for U-statistics},
  author={Pitcan, Yannik},
  journal={arXiv preprint arXiv:1712.06160},
  year={2017}
}

@article{wang2017g,
  title={G-computation of average treatment effects on the treated and the untreated},
  author={Wang, Aolin and Nianogo, Roch A and Arah, Onyebuchi A},
  journal={BMC medical research methodology},
  volume={17},
  number={1},
  pages={3},
  year={2017},
  publisher={Springer}
}

@article{shu2018improved,
  title={Improved estimation of average treatment effects on the treated: Local efficiency, double robustness, and beyond},
  author={Shu, Heng and Tan, Zhiqiang},
  journal={arXiv preprint arXiv:1808.01408},
  year={2018}
}

@article{rosenbaum1983central,
  title={The central role of the propensity score in observational studies for causal effects},
  author={Rosenbaum, Paul R and Rubin, Donald B},
  journal={Biometrika},
  volume={70},
  number={1},
  pages={41--55},
  year={1983},
  publisher={Oxford University Press}
}

@article{jagielski2020auditing,
  title={Auditing differentially private machine learning: How private is private sgd?},
  author={Jagielski, Matthew and Ullman, Jonathan and Oprea, Alina},
  journal={Advances in Neural Information Processing Systems},
  volume={33},
  pages={22205--22216},
  year={2020}
}

@article{kazmi2024panoramia,
  title={Panoramia: Privacy auditing of machine learning models without retraining},
  author={Kazmi, Mishaal and Lautraite, Hadrien and Akbari, Alireza and Tang, Qiaoyue and Soroco, Mauricio and Wang, Tao and Gambs, S{\'e}bastien and L{\'e}cuyer, Mathias},
  journal={Advances in Neural Information Processing Systems},
  volume={37},
  pages={57262--57300},
  year={2024}
}

@inproceedings{nasr2023tight,
  title={Tight auditing of differentially private machine learning},
  author={Nasr, Milad and Hayes, Jamie and Steinke, Thomas and Balle, Borja and Tram{\`e}r, Florian and Jagielski, Matthew and Carlini, Nicholas and Terzis, Andreas},
  booktitle={32nd USENIX Security Symposium (USENIX Security 23)},
  pages={1631--1648},
  year={2023}
}

@article{hudgens2008toward,
  title={Toward causal inference with interference},
  author={Hudgens, Michael G and Halloran, M Elizabeth},
  journal={Journal of the american statistical association},
  volume={103},
  number={482},
  pages={832--842},
  year={2008},
  publisher={Taylor \& Francis}
}

@article{DBLP:journals/popets/KulynychYCVT22,
  author       = {Bogdan Kulynych and
                  Mohammad Yaghini and
                  Giovanni Cherubin and
                  Michael Veale and
                  Carmela Troncoso},
  title        = {Disparate Vulnerability to Membership Inference Attacks},
  journal      = {Proc. Priv. Enhancing Technol.},
  volume       = {2022},
  number       = {1},
  pages        = {460--480},
  year         = {2022},
}

@article{Wager03072018,
author = {Stefan Wager and Susan Athey},
title = {Estimation and Inference of Heterogeneous Treatment Effects using Random Forests},
journal = {Journal of the American Statistical Association},
volume = {113},
number = {523},
pages = {1228--1242},
year = {2018},
}

@misc{wager2024causal,
  title={Causal inference: A statistical learning approach},
  author={Wager, Stefan},
  year={2024},
  publisher={Technical report, Stanford University}
}

@inproceedings{carlini2022membership,
  title={Membership inference attacks from first principles},
  author={Carlini, Nicholas and Chien, Steve and Nasr, Milad and Song, Shuang and Terzis, Andreas and Tramer, Florian},
  booktitle={2022 IEEE symposium on security and privacy (SP)},
  pages={1897--1914},
  year={2022},
  organization={IEEE}
}

@article{splawa1990application,
	title        = {On the application of probability theory to agricultural experiments. Essay on principles. Section 9.},
	author       = {Splawa-Neyman, Jerzy and Dabrowska, Dorota M and Speed, Terrence P},
	year         = {1990},
	journal      = {Statistical Science},
	publisher    = {JSTOR},
	pages        = {465--472}
}

@article{steinke2023privacy,
  title={Privacy auditing with one (1) training run},
  author={Steinke, Thomas and Nasr, Milad and Jagielski, Matthew},
  journal={Advances in Neural Information Processing Systems},
  volume={36},
  pages={49268--49280},
  year={2023}
}

@inproceedings{hardt2016train,
  title={Train faster, generalize better: Stability of stochastic gradient descent},
  author={Hardt, Moritz and Recht, Ben and Singer, Yoram},
  booktitle={International conference on machine learning},
  pages={1225--1234},
  year={2016},
  organization={PMLR}
}

@article{hu2022lora,
  title={Lora: Low-rank adaptation of large language models},
  author={Hu, Edward J and Shen, Yelong and Wallis, Phillip and Allen-Zhu, Zeyuan and Li, Yuanzhi and Wang, Shean and Wang, Liang and Chen, Weizhu and others},
  journal={ICLR},
  year={2022}
}

@article{miles2019causal,
  title={Causal inference when counterfactuals depend on the proportion of all subjects exposed},
  author={Miles, Caleb H and Petersen, Maya and van der Laan, Mark J},
  journal={Biometrics},
  volume={75},
  number={3},
  pages={768--777},
  year={2019},
  publisher={Oxford University Press}
}

@article{li2022random,
  title={Random graph asymptotics for treatment effect estimation under network interference},
  author={Li, Shuangning and Wager, Stefan},
  journal={The Annals of Statistics},
  volume={50},
  number={4},
  pages={2334--2358},
  year={2022},
  publisher={Institute of Mathematical Statistics}
}

@article{forastiere2024causal,
  title={Causal inference on networks under continuous treatment interference},
  author={Forastiere, Laura and Del Prete, Davide and Sciabolazza, Valerio Leone},
  journal={Social Networks},
  volume={76},
  pages={88--111},
  year={2024},
  publisher={Elsevier}
}

@article{clark2021approach,
  title={An approach to causal inference over stochastic networks},
  author={Clark, Duncan A and Handcock, Mark S},
  journal={arXiv preprint arXiv:2106.14145},
  year={2021}
}

@article{bousquet2002stability,
  title={Stability and generalization},
  author={Bousquet, Olivier and Elisseeff, Andr{\'e}},
  journal={Journal of machine learning research},
  volume={2},
  number={Mar},
  pages={499--526},
  year={2002}
}

@inproceedings{tople2020alleviating,
  title={Alleviating privacy attacks via causal learning},
  author={Tople, Shruti and Sharma, Amit and Nori, Aditya},
  booktitle={International Conference on Machine Learning},
  pages={9537--9547},
  year={2020},
  organization={PMLR}
}

@misc{noauthor_nyt_nodate,
	title = {{NYT} v. {OpenAI}: {The} {Times}’s {About}-{Face}},
    author = { {New York Times v. OpenAI}},
    shorttitle = {{NYT} v. {OpenAI}},
	url = {https://harvardlawreview.org/blog/2024/04/nyt-v-openai-the-timess-about-face/},
	language = {en-US},
	urldate = {2025-12-26},
	journal = {Harvard Law Review},
    year=2023
}

@misc{noauthor_concord_nodate,
	title = {Concord {Music} {Group}, {Inc}. v. {Anthropic} {PBC}, 5:24-cv-03811 - {CourtListener}.com},
    author = { {Concord Music Group v. Anthropic}},
	url = {https://www.courtlistener.com/docket/68889092/concord-music-group-inc-v-anthropic-pbc/},
	language = {en\_us},
	urldate = {2025-12-26},
	journal = {CourtListener},
    year=2023,
}

@misc{edpb,
	title = {Opinion 28/2024 on certain data protection aspects related to the processing of personal data in the context of AI models},
    author = { {European Data Protection Board}},
	url = {https://www.edpb.europa.eu/our-work-tools/our-documents/opinion-board-art-64/opinion-282024-certain-data-protection-aspects_en},
	language = {en-US},
	urldate = {2026-01-15},
    year=2024
}

@inproceedings{10.1145/3548606.3560694,
author = {Baluta, Teodora and Shen, Shiqi and Hitarth, S. and Tople, Shruti and Saxena, Prateek},
title = {Membership Inference Attacks and Generalization: A Causal Perspective},
year = {2022},
booktitle = {CCS},
}

@article{nasr2023scalable,
  title={Scalable extraction of training data from (production) language models},
  author={Nasr, Milad and Carlini, Nicholas and Hayase, Jonathan and Jagielski, Matthew and Cooper, A Feder and Ippolito, Daphne and Choquette-Choo, Christopher A and Wallace, Eric and Tram{\`e}r, Florian and Lee, Katherine},
  journal={arXiv preprint arXiv:2311.17035},
  year={2023}
}

@article{hayes2025strong,
  title={Strong membership inference attacks on massive datasets and (moderately) large language models},
  author={Hayes, Jamie and Shumailov, Ilia and Choquette-Choo, Christopher A and Jagielski, Matthew and Kaissis, George and Lee, Katherine and Nasr, Milad and Ghalebikesabi, Sahra and Mireshghallah, Niloofar and Sundaram Mutu Selva Annamalai, Meenatchi and others},
  journal={arXiv e-prints},
  pages={arXiv--2505},
  year={2025}
}

@article{barbero2025extracting,
  title={Extracting alignment data in open models},
  author={Barbero, Federico and Gu, Xiangming and Choquette-Choo, Christopher A and Sitawarin, Chawin and Jagielski, Matthew and Yona, Itay and Veli{\v{c}}kovi{\'c}, Petar and Shumailov, Ilia and Hayes, Jamie},
  journal={arXiv preprint arXiv:2510.18554},
  year={2025}
}

@article{zhang2024membership,
  title={Membership inference attacks cannot prove that a model was trained on your data},
  author={Zhang, Jie and Das, Debeshee and Kamath, Gautam and Tram{\`e}r, Florian},
  journal={arXiv preprint arXiv:2409.19798},
  year={2024}
}

@InProceedings{eichler_nob-mias_2025,
author="Eichler, C{\'e}dric
and Champeil, Nathan
and Anciaux, Nicolas
and Bensamoun, Alexandra
and H. Arcolezi, H{\'e}ber
and De Fuentes, Jos{\'e} Maria",
editor="Barhamgi, Mahmoud
and Wang, Hua
and Wang, Xin",
title="Nob-MIAs: Non-biased Membership Inference Attacks Assessment on Large Language Models with Ex-Post Dataset Construction",
booktitle="Web Information Systems Engineering -- WISE 2024",
year="2025",
publisher="Springer Nature Singapore",
address="Singapore",
pages="441--456",
isbn="978-981-96-0570-5"
}

@inproceedings{mattern-etal-2023-membership,
    title = "Membership Inference Attacks against Language Models via Neighbourhood Comparison",
    author = {Mattern, Justus  and
      Mireshghallah, Fatemehsadat  and
      Jin, Zhijing  and
      Sch{\"o}lkopf, Bernhard  and
      Sachan, Mrinmaya  and
      Berg-Kirkpatrick, Taylor},
    editor = "Rogers, Anna  and
      Boyd-Graber, Jordan  and
      Okazaki, Naoaki",
    booktitle = "Findings of the Association for Computational Linguistics: ACL 2023",
    month = jul,
    year = "2023",
    address = "Toronto, Canada",
    publisher = "Association for Computational Linguistics",
    url = "https://aclanthology.org/2023.findings-acl.719/",
    doi = "10.18653/v1/2023.findings-acl.719",
    pages = "11330--11343"
}

@inproceedings{maini_llm_nodate,
author = {Maini, Pratyush and Jia, Hengrui and Papernot, Nicolas and Dziedzic, Adam},
title = {LLM dataset inference: did you train on my dataset?},
year = {2024},
isbn = {9798331314385},
publisher = {Curran Associates Inc.},
address = {Red Hook, NY, USA},
booktitle = {Proceedings of the 38th International Conference on Neural Information Processing Systems},
articleno = {3941},
numpages = {24},
location = {Vancouver, BC, Canada},
series = {NIPS '24}
}

@misc{maini_dataset_2021,
    title = {Dataset {Inference}: {Ownership} {Resolution} in {Machine} {Learning}},
    shorttitle = {Dataset {Inference}},
    url = {http://arxiv.org/abs/2104.10706},
    doi = {10.48550/arXiv.2104.10706},
    language = {en},
    urldate = {2025-12-29},
    publisher = {arXiv},
    author = {Maini, Pratyush and Yaghini, Mohammad and Papernot, Nicolas},
    month = apr,
    year = {2021},
    note = {arXiv:2104.10706 [stat]},
}

@inproceedings{sablayrolles2019white,
  title={White-box vs black-box: Bayes optimal strategies for membership inference},
  author={Sablayrolles, Alexandre and Douze, Matthijs and Ollivier, Yann and Schmid, Cordelia and J{\'e}gou, Herv{\'e}},
  booktitle={ICML},
  year={2019}
}

@misc{sablayrolles_radioactive_2020,
    title = {Radioactive data: tracing through training},
    shorttitle = {Radioactive data},
    url = {http://arxiv.org/abs/2002.00937},
    doi = {10.48550/arXiv.2002.00937},
    language = {en},
    urldate = {2025-12-29},
    publisher = {arXiv},
    author = {Sablayrolles, Alexandre and Douze, Matthijs and Schmid, Cordelia and Jégou, Hervé},
    month = feb,
    year = {2020},
    note = {arXiv:2002.00937 [stat]},
}

@misc{kirchenbauer_watermark_2024,
    title = {A {Watermark} for {Large} {Language} {Models}},
    url = {http://arxiv.org/abs/2301.10226},
    doi = {10.48550/arXiv.2301.10226},
    language = {en},
    urldate = {2025-12-29},
    publisher = {arXiv},
    author = {Kirchenbauer, John and Geiping, Jonas and Wen, Yuxin and Katz, Jonathan and Miers, Ian and Goldstein, Tom},
    month = may,
    year = {2024},
    note = {arXiv:2301.10226 [cs]},
}

@inproceedings{he2016deep,
  title = {Deep Residual Learning for Image Recognition},
  booktitle = {2016 {{IEEE}} Conference on Computer Vision and Pattern Recognition ({{CVPR}})},
  author = {He, Kaiming and Zhang, Xiangyu and Ren, Shaoqing and Sun, Jian},
  year = {2016},
}

@INPROCEEDINGS{5206848,
  author={Deng, Jia and Dong, Wei and Socher, Richard and Li, Li-Jia and Kai Li and Li Fei-Fei},
  booktitle={2009 IEEE Conference on Computer Vision and Pattern Recognition}, 
  title={ImageNet: A large-scale hierarchical image database}, 
  year={2009},
  volume={},
  number={},
  pages={248-255},
  doi={10.1109/CVPR.2009.5206848}}

@article{krizhevsky2009learning,
  title={Learning multiple layers of features from tiny images},
  author={Krizhevsky, Alex and others},
  year={2009}
}

@inproceedings{biderman2023pythia,
  title={Pythia: A suite for analyzing large language models across training and scaling},
  author={Biderman, Stella and Schoelkopf, Hailey and Anthony, Quentin Gregory and Bradley, Herbie and O’Brien, Kyle and Hallahan, Eric and Khan, Mohammad Aflah and Purohit, Shivanshu and Prashanth, USVSN Sai and Raff, Edward and others},
  booktitle={International conference on machine learning},
  pages={2397--2430},
  year={2023},
  organization={PMLR}
}

@article{korner2022linguistic,
  title={How the linguistic styles of Donald Trump and Joe Biden reflect different forms of power},
  author={K{\"o}rner, Robert and Overbeck, Jennifer R and K{\"o}rner, Erik and Sch{\"u}tz, Astrid},
  journal={Journal of Language and Social Psychology},
  volume={41},
  number={6},
  pages={631--658},
  year={2022},
  publisher={Sage Publications Sage CA: Los Angeles, CA}
}
